\def\year{2020}\relax

\documentclass[letterpaper]{article} 
\usepackage{aaai20}  
\usepackage{times}  
\usepackage{helvet} 
\usepackage{courier}  
\usepackage[hyphens]{url}  
\usepackage{graphicx} 
\usepackage{graphicx}  
\title{Provably convergent acceleration in factored gradient descent \\ with applications in matrix sensing}
\author{Tayo Ajayi\textsuperscript{\rm 1}, David Mildebrath\textsuperscript{\rm 1}, Anastasios Kyrillidis\textsuperscript{\rm 1}\\ \Large \textbf{George Kollias\textsuperscript{\rm 2}, Shashanka Ubaru\textsuperscript{\rm 2}, Kris Bouchard\textsuperscript{\rm 3}} \\ 
\textsuperscript{\rm 1} Rice University \\ 
\textsuperscript{\rm 2} IBM T.J. Watson Research Center \\ 
\textsuperscript{\rm 3 }Lawrence Berkeley National Laboratory 
}

\usepackage{amsmath,amsfonts,amsthm,amssymb,xspace,bm}
\usepackage{verbatim,dsfont,mathtools,algorithm,algorithmic,url}
\usepackage{booktabs}

\usepackage{multirow}
\usepackage{xfrac}
\usepackage{wrapfig}
\usepackage{varwidth}
\usepackage{colortbl}
\usepackage{epstopdf}
\usepackage{appendix}
\usepackage{enumerate}
\usepackage{pifont}
\usepackage[table]{xcolor}
\usepackage{color}

\theoremstyle{plain}
\newtheorem{corollary}{Corollary}
\newtheorem{definition}{Definition}
\newtheorem{remark}{Remark}

\DeclareMathOperator{\trace}{\textsc{Tr}}

\newcommand\RR{{\mathbb R}}
\newcommand{\tvec}{\mathtt{vec}}
\newcommand{\mbf}{\mathbf}
\usepackage{enumitem}
\usepackage{sidecap}

\newtheorem{theorem}{Theorem}

\newtheorem{lemma}{Lemma}

\begin{document}

\maketitle

\begin{abstract}
We present theoretical results on the convergence of \emph{non-convex} accelerated gradient descent in matrix factorization models with $\ell_2$-norm loss.
The purpose of this work is to study the effects of acceleration in non-convex settings, where provable convergence with acceleration should not be considered a \emph{de facto} property.
The technique is applied to matrix sensing problems, for the estimation of a rank $r$ optimal solution $X^\star \in \mathbb{R}^{n \times n}$. 
Our contributions can be summarized as follows.
$i)$ We show that acceleration in factored gradient descent converges at a linear rate; this fact is novel for non-convex matrix factorization settings, under common assumptions. 
$ii)$ Our proof technique requires the acceleration parameter to be carefully selected, based on the properties of the problem, such as the condition number of $X^\star$ and the condition number of objective function. 
$iii)$ Currently, our proof leads to the same dependence on the condition number(s) in the contraction parameter, similar to recent results on non-accelerated algorithms.
$iv)$ Acceleration is observed in practice, both in synthetic examples and in two real applications: neuronal multi-unit activities recovery from single electrode recordings, and quantum state tomography on quantum computing simulators.
\end{abstract}

Accelerated versions of gradient descent (GD), inspired by Polyak~\cite{polyak1964some} and Nesterov \cite{nesterov1983method}, are the methods  of choice in various optimization tasks, including training  deep neural networks \cite{sutskever2013importance,szegedy2015going}. 
Acceleration is based on \emph{momentum}:  as long as the iterates point to (approximately) the same direction, momentum favors the sequence of future estimates along that path.
In this way, momentum leads to empirically faster decrease in the objective function~\cite{nesterov1983method}.

Despite its widespread use \cite{kingma2014adam,tieleman2012lecture,kingma2014adam}, theoretical results on why momentum works well
are mostly restricted to the convex case, where it provably begets significant gains with respect to convergence rate \cite{beck2009fast,o2015adaptive,bubeck2015geometric,goh2017why}.
Exceptions include $i)$ settings that involve non-convex (and structured) constraint sets \cite{kyrillidis2011recipes,kyrillidis2014matrix,khanna2017iht,xu2018accelerated}, but with a convex objective;
and $ii)$ papers that consider generic non-convex settings, but do not focus on finding the global solution: they study whether acceleration leads to fast convergence to a critical point --saddle point or local minimum \cite{ghadimi2013stochastic,lee2016gradient,carmon2016accelerated,agarwal2016finding}.
In the latter case, we also observe the difficulty of achieving acceleration in theory: the gains in theory are restricted to obtaining an improved rate from $O\left(\tfrac{1}{\varepsilon^2}\right)$ to $O\left(\tfrac{1}{\varepsilon^{c}}\right)$, where $c$ is less but close to 2.

In this work, we study momentum theoretically in the context of shallow, linear neural networks, using low-rank factorization for the matrix sensing problem as our test case. 
Such simplifications have been followed in other recent works in machine learning and theoretical computer science, such as the cases of convolutional neural networks \cite{du2017gradient},
the effect of over-parameterization in training \cite{li2017algorithmic}, and landscape characterization of generic objectives \cite{boob2017theoretical,safran2017spurious}.
Our work can be seen as a first step towards understanding momentum in general non-linear models, whose training objectives are more involved and complex.

Our contributions can be summarized as follows:
\begin{itemize}[leftmargin=0.5cm]
\item For matrix sensing, we prove that a heavy-ball-like method, operating on the low-rank factors, converges linearly to the optimal solution, \emph{up to some error level that depends on the acceleration parameter}. 
Our proof technique carefully adapts ideas from the two-step momentum approach in the convex setting to the non-convex setting and requires careful selection of the step size and momentum parameters.
The convergence proof differs from that of non-accelerated methods, due to the inclusion of history of estimates per iteration. 
Our theory requires assumptions on the condition number of the objective and that of the optimal solution. 
This expands the recent results on the favorable performance of non-convex algorithms over convex methods.
\item We provide empirical evidence of the convergence of accelerated gradient descent in non-convex settings, both for synthetic and actual engineering problem settings. 
For the latter, we focus on two applications. 
$i)$ The task of identifying neuronal activities located at different depths of the brain: the idea is that the observations at the surface of the brain are linear combinations of these neural activities, lowpass filtered and attenuated as per their depths. 
$ii)$ The task of quantum state tomography (QST) using quantum computing simulators \cite{IBMQiskit}: pure quantum states are naturally represented as low-rank density matrices, and the QST becomes computationally prohibitive as the number of qubits increases. 
\end{itemize}
We identify several remaining open questions that need to be resolved. 
First, the range of values for the momentum parameter that comply with our theory is conservative (though even in this conservative regime, we still observe empirical speed-up over non-accelerated methods). We also require problem instances with fairly well-conditioned optimal solutions. 
Moreover, while we prove linear convergence when acceleration is used, the same dependence on the condition number(s) is observed as in \cite{tu2016low}; whether this can be improved is an open question.
It is also not clear whether the additive error (the $O(\mu)$ term in Theorem~\ref{thm:00}) in our proof is necessary, or is an artifact of our proof technique.
Finally, the analysis of the matrix sensing with noise is left for future work.

\section{Setup}
The matrix sensing problem has been studied  extensively in the literature; see \cite{recht2010guaranteed} and references to it.
It has numerous applications including video background subtraction~\cite{waters2011sparcs}, system approximation~\cite{fazel2001rank} and identification~\cite{liu2009interior}, robust PCA~\cite{candes2011robust}, and quantum state tomography~\cite{flammia2012quantum}.

For clarity, we consider the matrix sensing problem for square matrices $X \in \mathbb{R}^{n \times n}$, under both low rank and positive semi-definite (PSD) constraints \cite{bhojanapalli2016dropping,bhojanapalli2016global,li2017algorithmic}:\footnote{The rectangular case can be derived, after proper transformations, using ideas from \cite{park2016finding}.}
\begin{equation}\label{eq:obj}
\begin{aligned}
& \min_{X \in \mathbb{R}^{n \times n}}
& & f(X) := \tfrac{1}{2} \|\mathcal{A}(X) - y\|_2^2 \\
& \text{subject to}
& & X \succeq 0, ~\texttt{rank}(X) \leq r.
\end{aligned}
\end{equation}
$X$ is the decision variable that lives at the intersection of low-rank and PSD constraints; $y \in \mathbb{R}^m$ is the set of observations; and $\mathcal{A}(\cdot): \mathbb{R}^{n \times n} \rightarrow \mathbb{R}^m$ is the linear sensing map, where $m \ll n^2$.
We take $\mathcal{A}$ to be the trace operator, given by $\left(\mathcal{A}(X) \right)_i = \texttt{Tr}(A_i^\top X)$, for symmetric random $A_i \in \mathbb{R}^{n \times n}$ and $i=1,\dots,m$. 

A pivotal assumption is that $\mathcal{A}$ satisfies the \emph{restricted isometry property}:
\begin{definition}[Restricted Isometry Property (RIP) \cite{recht2010guaranteed}]\label{def:rip}
A linear operator $\mathcal{A} :~\mathbb{R}^{n \times n} \rightarrow \mathbb{R}^m$ satisfies the RIP on rank-$r$ matrices, with parameter  $\delta_{r} \in (0, 1)$, if the following holds for all rank-$r$ $X$:
\begin{align*}
(1 - \delta_r) \cdot \|X\|_F^2 \leq \|\mathcal{A}(X)\|_2^2 \leq (1 + \delta_r) 
\cdot \|X\|_F^2.
\end{align*}
\end{definition}

Recent works focus on the factorized version of the problem, due to time/space complexity savings:
\begin{equation}\label{eq:factobj}
\min_{U \in \mathbb{R}^{n \times r}} ~\tfrac{1}{2} \|\mathcal{A}(UU^\top) - y\|_2^2;
\end{equation}
see also Section \ref{sec:related} for a subset of references on the subject.
Observe that any matrix $X \succeq 0$ with $\texttt{rank}(X) \leq r$, can be written as $X = UU^\top$, for $U \in \mathbb{R}^{n \times r}$; this re-parameterization encapsulates both constraints in \eqref{eq:obj} leading to the non-convex formulation \eqref{eq:factobj}.
A common approach to solve \eqref{eq:factobj} is to use gradient descent, with iterates generated by the rule:\footnote{We assume cases where $\nabla f(\cdot) = \nabla f(\cdot)^\top$. 
If this does not hold, the theory goes through by carrying around 
$\nabla f(\cdot) + \nabla f(\cdot)^\top$ instead of just $\nabla f(\cdot)$,
after proper scaling.}
\begin{align*}
U_{i+1} &= U_{i} - \eta \nabla f(U_i U_i^\top) \cdot U_i \\ &= U_{i} - \eta \mathcal{A}^\dagger \left(\mathcal{A}(U_i U_i^\top) - y\right) \cdot U_i.
\end{align*}
The operator $\mathcal{A}^\dagger :~\mathbb{R}^{m} \rightarrow \mathbb{R}^{n \times n}$
is the adjoint of $\mathcal{A}$, defined as $\mathcal{A}^\dagger(x) = \sum_{i = 1}^m x_i A_i$, for $x \in \mathbb{R}^m$. The paramter $\eta>0$ is a step size.
This algorithm has been studied in \cite{bhojanapalli2016dropping,zheng2015convergent,tu2016low,park2016non,ge2017no,hsieh2017non}. 
We will refer to the above recursion as the \emph{Procrustes Flow} algorithm, as in \cite{tu2016low}. 
None of the above works have considered momentum acceleration on $U$.
In this work, we will study the accelerated version of Procrustes Flow. 

\section{Accelerated Procrustes Flow and Main Results}
We consider the following two-step variant of Procrustes flow:
\begin{align*}
U_{i+1} &= Z_{i} - \eta \mathcal{A}^\dagger \left(\mathcal{A}(Z_i Z_i^\top) - y\right) \cdot Z_i, \\
Z_{i+1} &= U_{i+1} + \mu \left(U_{i+1} - U_i\right).
\end{align*}
Here, $Z_i$ is an auxiliary variable that accumulates the ``momentum" of the iterates $U_i$; the dimensions are apparent from the context. 
$\mu$ is the momentum parameter that weighs how the previous estimates $U_i$ will be mixed with the current estimate $U_{i+1}$ to generate $Z_{i+1}$.

The above recursion is an adaptation of Nesterov's accelerated first-order method for convex problems \cite{nesterov1983method}:
briefly, consider the generic convex optimization problem
$\min_{x \in \mathbb{R}^d} g(x)$,
where $g$ is a convex function that satisfies standard Lipschitz gradient continuity assumptions with Lipschitz constant $L$.
Nesterov's accelerated method is given by the recursion:
$x_{i+1} = y_i - \tfrac{1}{L} \nabla g(y_i), ~ \text{and} ~ y_{i+1} = x_{i+1} + \mu_{i} (x_{i+1} - x_i)$,
where the parameters $\mu_i$ are chosen to obey specific rules (see \cite{nesterov2013introductory} for more details).
We borrow this momentum formulation, but we study \emph{how constant $\mu$ selections behave in non-convex problem formulations}, such as in \eqref{eq:factobj}.
We note that the theory and algorithmic configurations in \cite{nesterov1983method} do not trivially generalize to non-convex problems.

\medskip
\noindent \textbf{Preliminaries.} 
An important observation for any factorization $X = UU^\top$ is that it is not unique.
That is, if $U^\star$ is an optimal solution for~\eqref{eq:factobj}, then for any matrix $R\in\mathbb{R}^{r\times{r}}$ satisfying $R^\top{R}=I$, the matrix $\widehat{U}=U^\star{R}$ is also optimal for \eqref{eq:factobj}, because $\widehat{U} \widehat{U}^\top = U^\star R \cdot R^{\top} U^{\star \top} = U^\star U^{\star \top}$.
To resolve this ambiguity, we define the distance between a pair of matrices as the minimum distance $\min_{R \in \mathcal{O}} \left \| U  - U^\star R \right\|_F$, where $\mathcal{O}=\{R\in\mathbb{R}^{r\times{r}}\;|\;R^\top{R}=I\}$.
\begin{algorithm}
\begin{algorithmic}
   \STATE {\bfseries Input:} $\mathcal{A}$, $y$,  $r$, $\mu$, and  $\#$ iterations $J$.
  \STATE Set $\eta$ as in \eqref{eq:step}.
  \STATE Set $U_0$ randomly or according to Lemma \ref{lem:init}, and $Z_0=U_0$.
\FOR{$i=0$ to $J-1$}
\STATE  $U_{i+1} = Z_i  - \eta  \mathcal{A}^\dagger \left( \mathcal{A}(ZZ^\top) - y \right) \cdot Z_i$ \\
\STATE   $Z_{i+1} = U_{i+1} + \mu \left(U_{i+1} - U_i\right)$
\ENDFOR
 \STATE {\bfseries Output:} $X=U_{J} U_{J}^\top$
\end{algorithmic}
\caption{Accelerated Procrutes Flow} \label{alg:algo1}
\vspace{-0.1cm}
\end{algorithm}

\noindent \textbf{The algorithm.}
Algorithm \ref{alg:algo1} contains the details of the Accelerated Procrustes Flow.
The algorithm requires as input the target rank\footnote{In this work, we assume we know the target rank \emph{a priori}. For the cases where we undershoot the rank, the theory from \cite{park2016finding} can be used to extend our theory.} $r$, the number of iterations $J$, and the momentum parameter $\mu$. 
For \emph{our theory} to hold, we make the following selections.
$i)$ The $\mu$ selection is conservative as we show next, but more aggressive $\mu$ choices lead to different requirements in our theory.
$ii)$ We use step size:
\begin{align}
\eta = \tfrac{1}{4 \left( (1 + \delta_{2r})\|Z_0Z_0^\top\|_2 +\|\mathcal{A}^\dagger \left(\mathcal{A}(Z_0Z_0^\top) - y \right)\|_2 \right)}, \label{eq:step}
\end{align}
where $Z_0 = U_0$. 
Observe that $\eta$ remains constant throughout the execution, and requires two top-eigenvalue computations: that of $Z_{0}Z_{0}^{T}$ and $\mathcal{A}^\dagger \left( \mathcal{A}(Z_0 Z_0^\top - y\right)$. 
Our experiments show that this step can be efficiently implemented by any off-the-shelf eigenvalue solver, such as the Power Method or the Lanczos method.
$iii)$ The initial point $U_0$ is either randomly selected \cite{bhojanapalli2016global, park2016non}, or set according to the following Lemma:

\begin{lemma}[\cite{kyrillidis2017provable}]{\label{lem:init}}
Let $U_0$ such that $X_0 = U_0 U_0^\top = \Pi_{\mathcal{C}}\big(\tfrac{-1}{1 + \delta_{2r}} \cdot \nabla f(0_{n \times n}) \big)$, where $\Pi_{\mathcal{C}}(\cdot)$ is the projection onto the set of PSD matrices.
Consider the matrix sensing problem with RIP for some constant $\delta_{2r} \in (0, 1)$. 
Further, assume the optimal point $X^\star$ satisfies $\text{rank}(X^\star) = r$. 
Then
the initial point $U_0$ satisfies:
\begin{align*}
\min_{R \in \mathcal{O}} \|U_0 - U^\star R\|_F \leq \gamma' \cdot \sigma_r(U^\star),
\end{align*} 
where $\gamma' = \sqrt{\tfrac{1 - \tfrac{1- \delta_{2r}}{1 + \delta_{2r}}}{2(\sqrt{2}-1)}} \cdot \tau(X^\star) \cdot \sqrt{\texttt{srank}(X^\star)}$, $\tau(X^\star) = \tfrac{\sigma_1(X^\star)}{\sigma_r(X^\star)}$, and $\texttt{srank}(X) = \tfrac{\|X\|_F}{\sigma_1(X)}$.
\end{lemma}
\noindent Since computing the RIP constants is NP-hard, in practice we compute $X_0 = \Pi_{\mathcal{C}}\big(-1 / \widehat{L} \cdot \nabla f(0) \big)$, where $\widehat{L} \in (1,2)$. 
We do not know \emph{a priori} $\tau(X^\star)$ and $\texttt{srank}(X^\star)$ to compute $\gamma'$, but they can be approximated depending on the problem at hand; \emph{e.g.}, in the quantum state tomography case, the rank could be $r=1$ (for pure quantum states), and we know apriori that $\tau(X^\star) = \texttt{srank}(X^\star) = 1$, by construction.
Compared to randomly selecting $U_0$, Lemma \ref{lem:init} involves a gradient descent computation and a top-$r$ eigenvalue calculation.
While randomly selecting $U_0$ guarantees convergence \cite{bhojanapalli2016global,park2016non}, Lemma \ref{lem:init} provides the initial conditions for our theory also leads to convergence rate guarantees.

\medskip
\noindent \textbf{Main theorem.}
The following theorem proves the convergence of accelerated Procrustes Flow under assumptions on $\mu$, the RIP constant $\delta_{2r}$, the condition number of the objective $\kappa$ (defined below) and that of the optimal solution $\tau(X^\star)$, using the initialization in Lemma \ref{lem:init}.
We note that such assumptions are needed in order to provide a concrete and qualitative convergence result.

\medskip
\begin{theorem}[Iteration invariant and convergence rate]{\label{thm:00}}
Assume that $\mathcal{A}$ satisfies the RIP with constant $\delta_{2r} \leq \sfrac{1}{10}$.
Let $U_0$ be such that \begin{small}$\min_{R \in \mathcal{O}} \|U_0 -U^\star R\|_F \leq \tfrac{\sigma_r(X^\star)^{1/2}}{10^3 \sqrt{\kappa\tau(X^\star)}}$\end{small},
where \begin{small}$\kappa := \tfrac{1 + \delta_{2r}}{1 - \delta_{2r}}$\end{small} and \begin{small}$\tau(X) := \tfrac{\sigma_1(X)}{\sigma_r(X)}$\end{small} for rank-$r$ $X$.
Set $\eta$ according to \eqref{eq:step} and the momentum parameter \begin{small}$\mu = \frac{\sigma_r(X^\star)^{1/2}}{10^3 \sqrt{\kappa\tau(X^\star)}} \cdot \frac{\varepsilon}{2 \cdot \sigma_1(X^\star)^{1/2} \cdot r}$\end{small}, for $\varepsilon \in (0, 1]$.
For $y = \mathcal{A}(X^\star)$, where \texttt{rank}$(X^\star)=r$, if $\tau(X^\star)\leq{50}$, then Algorithm~\ref{alg:algo1} returns a solution such that
\begin{small}
\begin{align*}
\min_{R \in \mathcal{O}} \|U_{J+1} - U^\star R\|_F \leq \alpha c^{J+1} \cdot \min_{R \in \mathcal{O}} \|U_0 - U^\star R\|_F + O(\mu).
\end{align*}
\end{small}
where
$c := |\lambda_1| < 1$, $\alpha := \tfrac{4}{|\lambda_1| - |\lambda_2|} \cdot \big(1 + \tfrac{\xi^2 (1 + \tfrac{1}{10^3})}{1 - \xi(1 + \tfrac{1}{10^3})} \cdot \tfrac{\varepsilon}{2}\big)$, $\lambda_i$ are the eigenvalues of $A = \begin{small}\begin{bmatrix} \xi |1 + \mu| & \xi |\mu| \\ 1 & 0 \end{bmatrix}\end{small}$, and $\xi := \sqrt{1 - \tfrac{0.393}{\kappa \tau(X^\star)}}$.
That is, the algorithm has a linear convergence rate in iterate distances (first term on RHS), up to a constant proportional to the the momentum hyperparameter $\mu$ (second term on RHS). 

Further, $U_i$ satisfies \begin{small}$\min_{R \in \mathcal{O}} \|U_i - U^\star R\|_F \leq \tfrac{\sigma_r(X^\star)^{1/2}}{10^3 \sqrt{\kappa\tau(X^\star)}}$\end{small}, for each $i \geq 1$.
\end{theorem}

The intuition is that the right hand side of the recursion: $i)$ depends on the initial distance $\min_{R \in \mathcal{O}} \|U_0 - U^\star R\|_2$, as in convex optimization, $ii)$ there are two parameters $c, \alpha$ (which are both constants) that appear as contraction constants, and $iii)$ $c$ drops exponentially fast to zero, since $|\lambda_1| < 1$, which means that that the product $\alpha c^{J+1}$ goes exponentially fast to zero.

$\alpha$ depends on the spectral gap of the contraction matrix $A$, as well as on $\xi$, where the latter depends on the condition number of the objective and the condition number of $X^\star$.

Finally, the theory generates an additional term $O(\mu)$.
Similar results exist in the literature (see, for example, the constant step size convergence of convex SGD, where one achieves linear convergence up to an error level that depends on the step size).  

The detailed proof is provided in the supplementary material.
To the best of our knowledge, this is the first proof for accelerated factored gradient descent, under common assumptions, both regarding the problem setting, and the assumptions made for its completion. 
The proof differs from state of the art proofs for non-accelerated factored gradient descent: due to the inclusion of the memory term, three different terms --$U_{i+1}, U_i, U_{i-1}$-- need to be handled simultaneously. 
Further, the proof differs from recent proofs on non-convex, but non-factored, gradient descent methods, as in \cite{khanna2017iht}: the distance metric over rotations $\min_{R \in \mathcal{O}} \|Z_i - U^\star R\|_F$, where $Z_i$ includes estimates from two steps in history, is not amenable to simple triangle inequality bounds, and a careful analysis is required.
Finally, the analysis requires the design of two-dimensional dynamical systems, where we require to characterize and bound the eigenvalues of a $2\times 2$ contraction matrix.

\begin{remark}
Tighter analysis w.r.t. constants would result into milder assumptions in the theorem; the choice of constants is made just for proof of concept.  We note that even under the conservative assumptions of Theorem \ref{thm:00}, we observe empirical speedup over non-accelerated methods.
\end{remark}

\begin{remark}
Observe that the main results depend on $\tau(X^\star)$ and $\kappa$;\footnote{This is not obvious from this result, but it is shown in the full proof, where we bound $\kappa$ using $\delta_{2r} \leq \sfrac{1}{10}$.} classic results from convex optimization depend on the $\sqrt{\kappa}$. 
It remains an open question whether provable acceleration with dependence on at least $\sqrt{\kappa}$ can be achieved using our method. 
On the other hand, we also note that it is not generally known whether the acceleration guarantee generalizes to all functions in convex optimization \cite{lessard2016analysis}. 
However, acceleration is observed in practice.
\end{remark}

\begin{remark}
Different choices of ($\tau(X^\star)$, $\mu$, $\delta_{2r}$) lead to different interdependences.
The current proof requires a strong assumption on $\mu$, which depends on quantities that might not be known a priori. 
However, as we note above, there are applications where knowing or approximating $\tau(X^\star), \texttt{srank}(X^\star),$ and $\kappa$ can help setting $\mu$. 
Further, as we show in the computational experiments, the theory is conservative; a much larger $\mu$ leads to stable, improved performance.
\end{remark}

\section{Related Work}\label{sec:related}
Matrix sensing was first studied in the convex setting using nuclear norm minimization~\cite{recht2010guaranteed,lee2009guaranteed,liu2009interior}.
Non-convex approaches involving rank-constraints have been proposed in \cite{jain2010guaranteed,lee2010admira,kyrillidis2014matrix}. 
In both cases, the algorithms involve a full or at least a truncated SVD per iteration. 
General low-rank minimization problems using the non-convex factorized formulation have been studied recently, due to  computational and space complexity 
advantages~\cite{jain2013provable,chen2015fast,zhao2015nonconvex,park2016provable,park2016non,sun2016guaranteed,bhojanapalli2016dropping,bhojanapalli2016global,park2016finding,ge2017no,hsieh2017non,kyrillidis2017provable}.
The factorized version was popularized in solving semi-definite programming~\cite{burer2003nonlinear}. 
Using factorization in matrix sensing, the \emph{Procrustes Flow} approach was studied in~\cite{tu2016low,zheng2015convergent}, with certain initializations techniques, different from the current work: we rely on a unique top-$r$ SVD computation, instead of multiple ones. 
\emph{None of the works above consider the analysis of acceleration in the proposed methods.}

\section{Experiments}
\subsection{Synthetic experiments}{\label{sec:synthetic}}
In this set of experiments, we compare Accelerated Procrustes Flow with 
$i)$ the Matrix ALPS framework \cite{kyrillidis2014matrix}, \emph{a projected gradient descent algorithm}---an optimized version of matrix IHT--- operating on the full matrix variable $X$, with adaptive step size $\eta$ (we note that this algorithm has outperformed most of the schemes that work on the original space $X$)
$ii)$ the plain Procrustes Flow algorithm \cite{tu2016low}, where we use the step size as reported in \cite{bhojanapalli2016dropping}, since the later has reported better performance than vanilla Procrustes Flow.
\emph{We note that the Procrustes Flow algorithm is the same as our algorithm without acceleration.}
Further, the original Procrustes Flow algorithm relies on performing many iterations in the original space $X$ as an initialization scheme, which is often prohibitive as the problem dimensions grow. 
Both for our algorithm and the plain Procrustes Flow scheme, we use both random initializations, as well as specific initializations according to Lemma \ref{lem:init}; this is also supported by the work \cite{bhojanapalli2016dropping,park2016finding}. 
In that case Procrustes Flow algorithm is identical to the Factored Gradient Descent (\texttt{FGD}) algorithm in \cite{bhojanapalli2016dropping,park2016finding}.

To properly compare the algorithms in the above list, we pre-select a common set of problem parameters.
We fix the dimension $n = 4096$ and the rank of the optimal matrix $X^\star \in \mathbb{R}^{n \times n}$ to be $r = 10$; similar behavior has been observed for other values of $r$, and are omitted.
We fix the number of observables $m$ to be $m = c \cdot n \cdot r$, where $c \in \{3, 5\}$.
In all algorithms, we fix the maximum number of iterations to 4000, and we use the same stopping criterion: \begin{scriptsize}$\|X_{i+1} - X_i\|_F / \|X_i\|_F \leq \texttt{tol}=10^{-3}$\end{scriptsize}. 

\begin{figure*}[ht]
\centering
\includegraphics[width=0.31\textwidth]{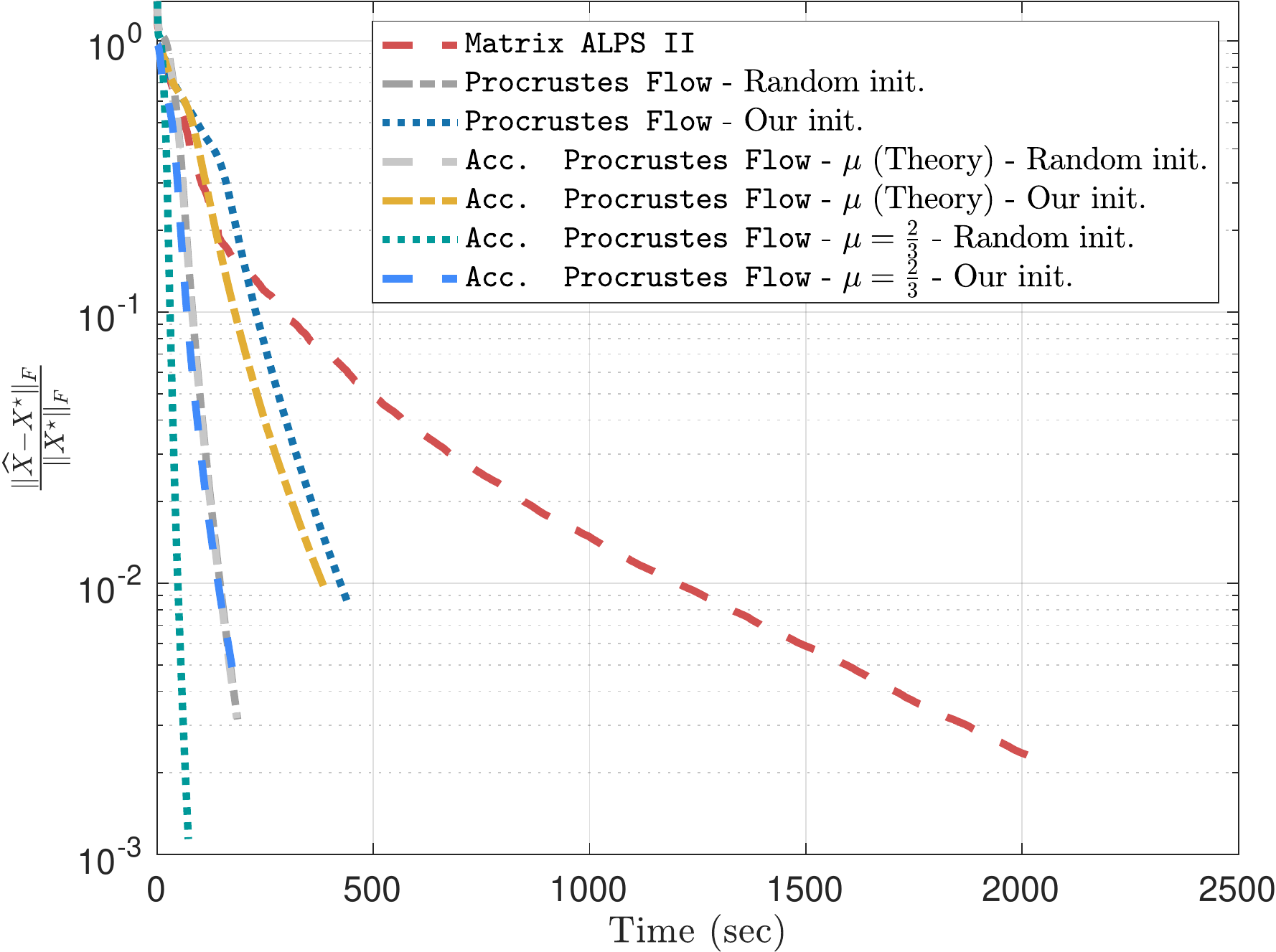} \includegraphics[width=0.32\textwidth]{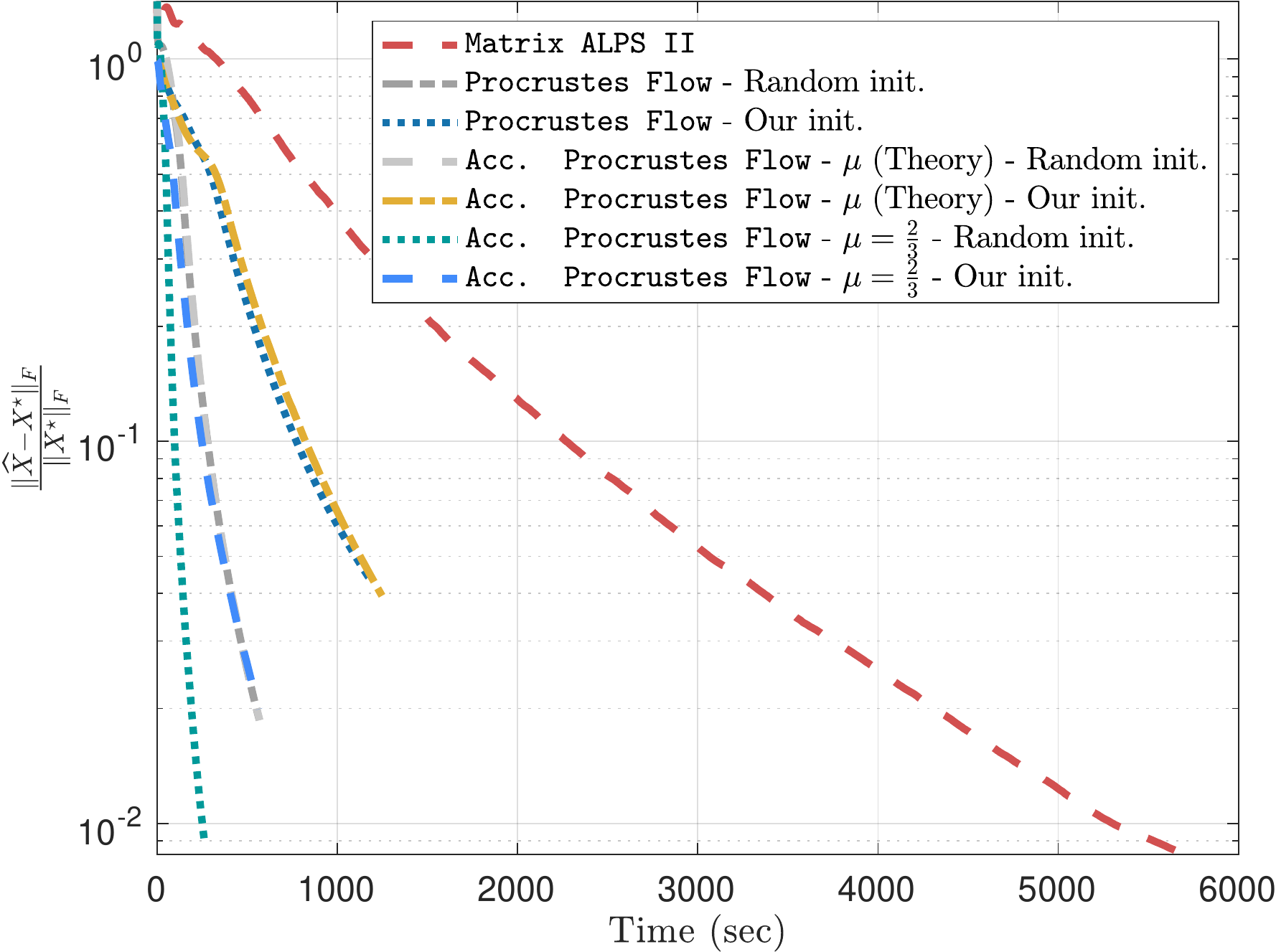} \includegraphics[width=0.32\textwidth]{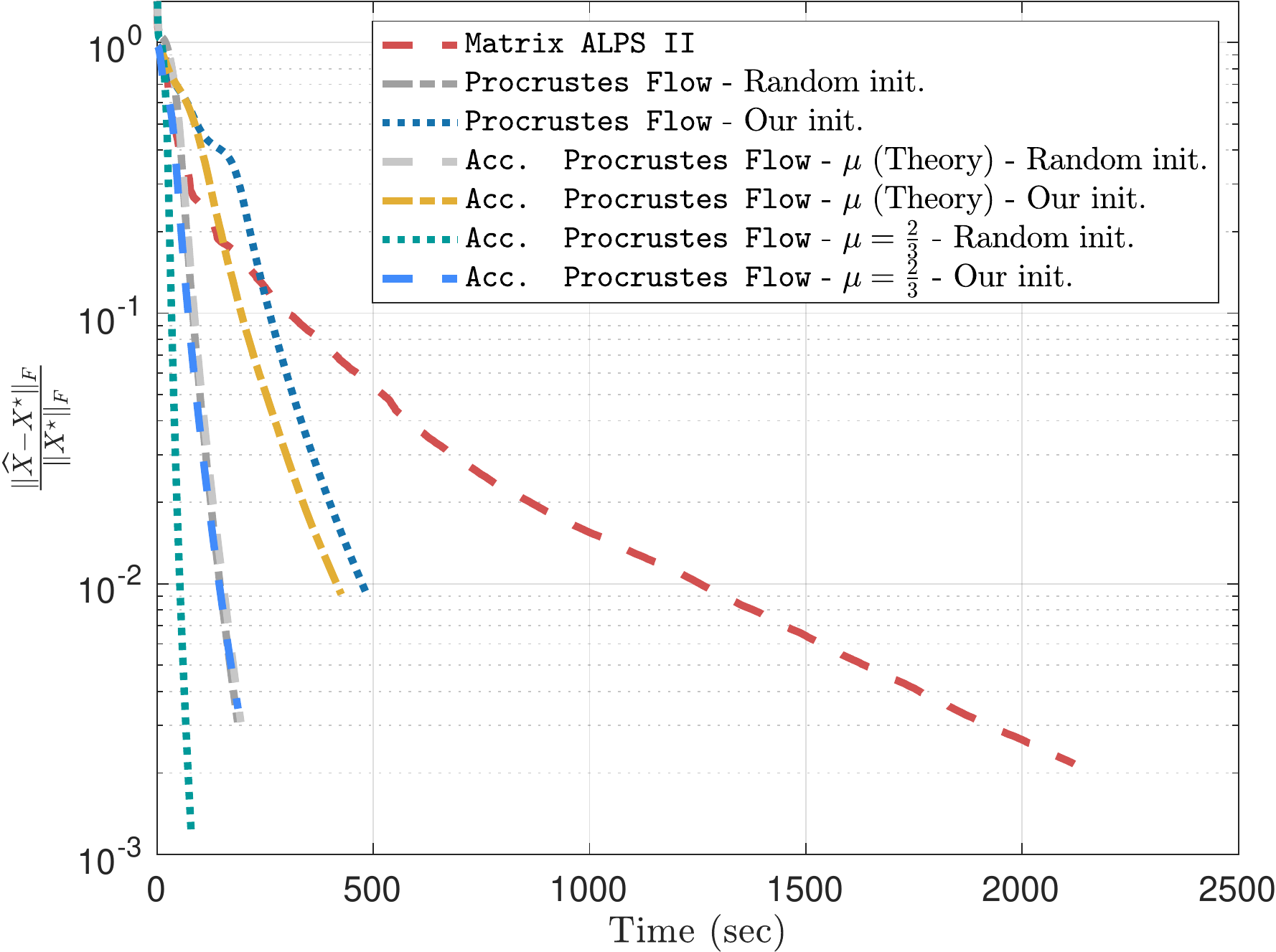}  \\
\includegraphics[width=0.31\textwidth]{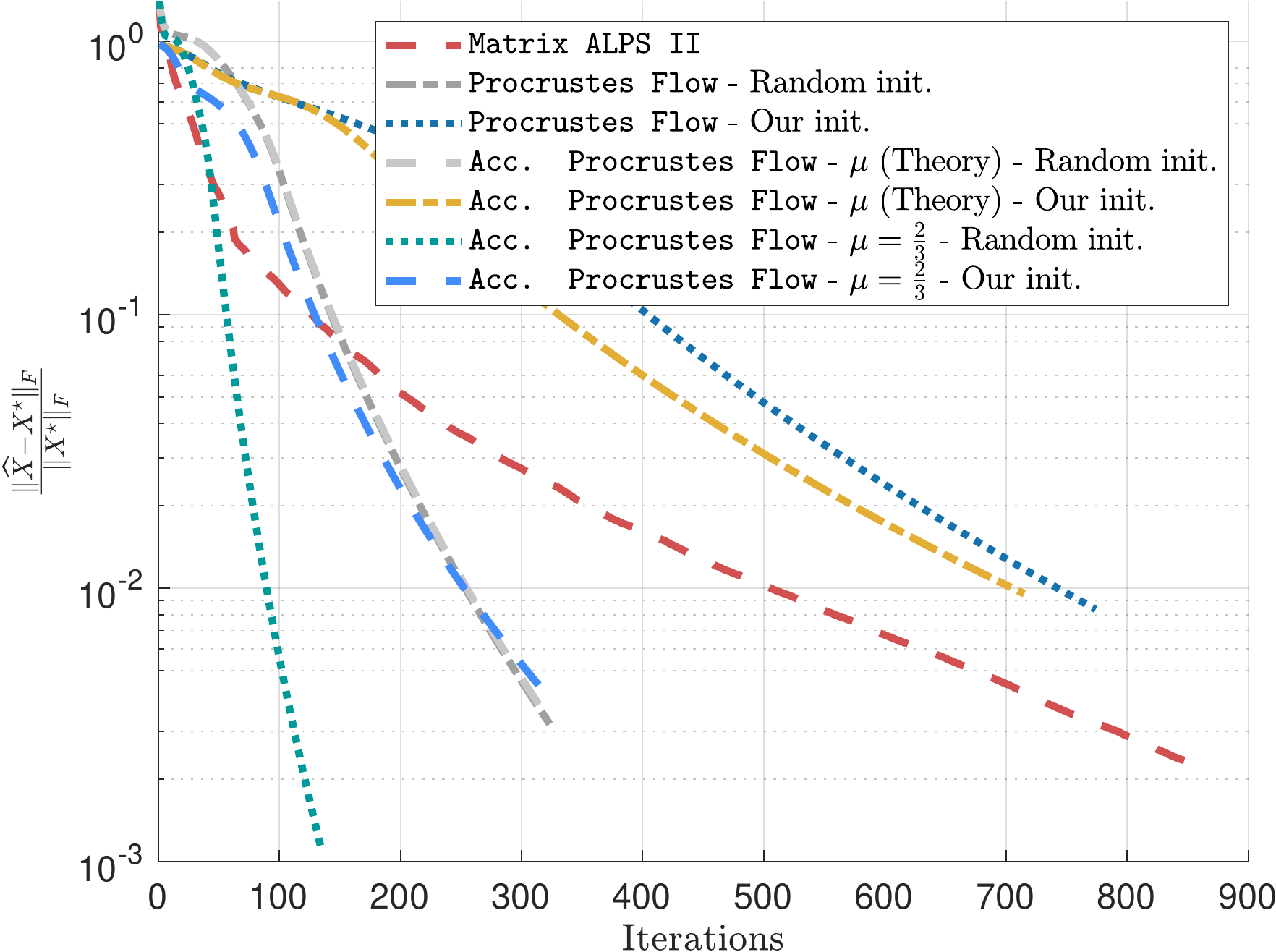} \includegraphics[width=0.32\textwidth]{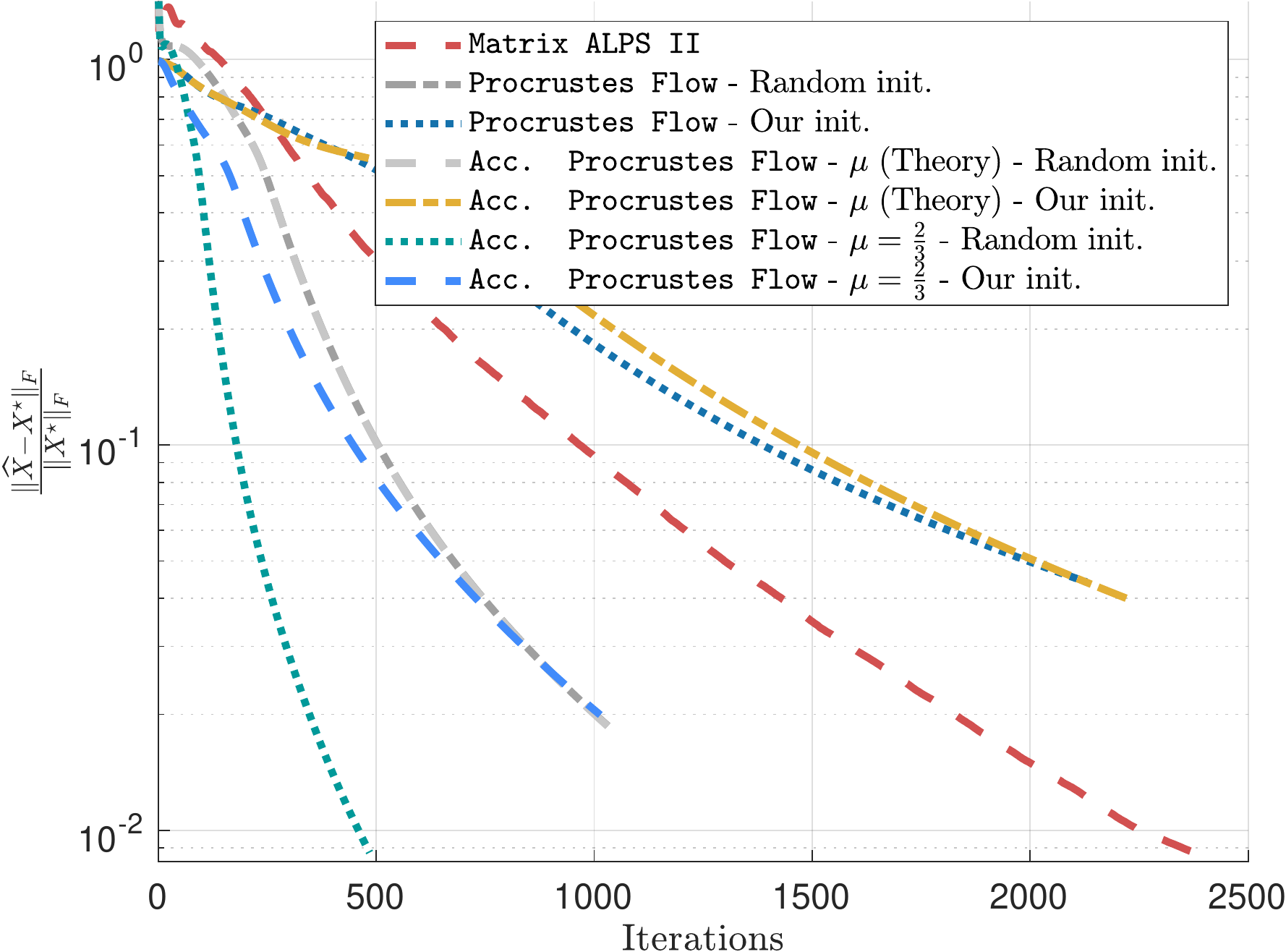} \includegraphics[width=0.32\textwidth]{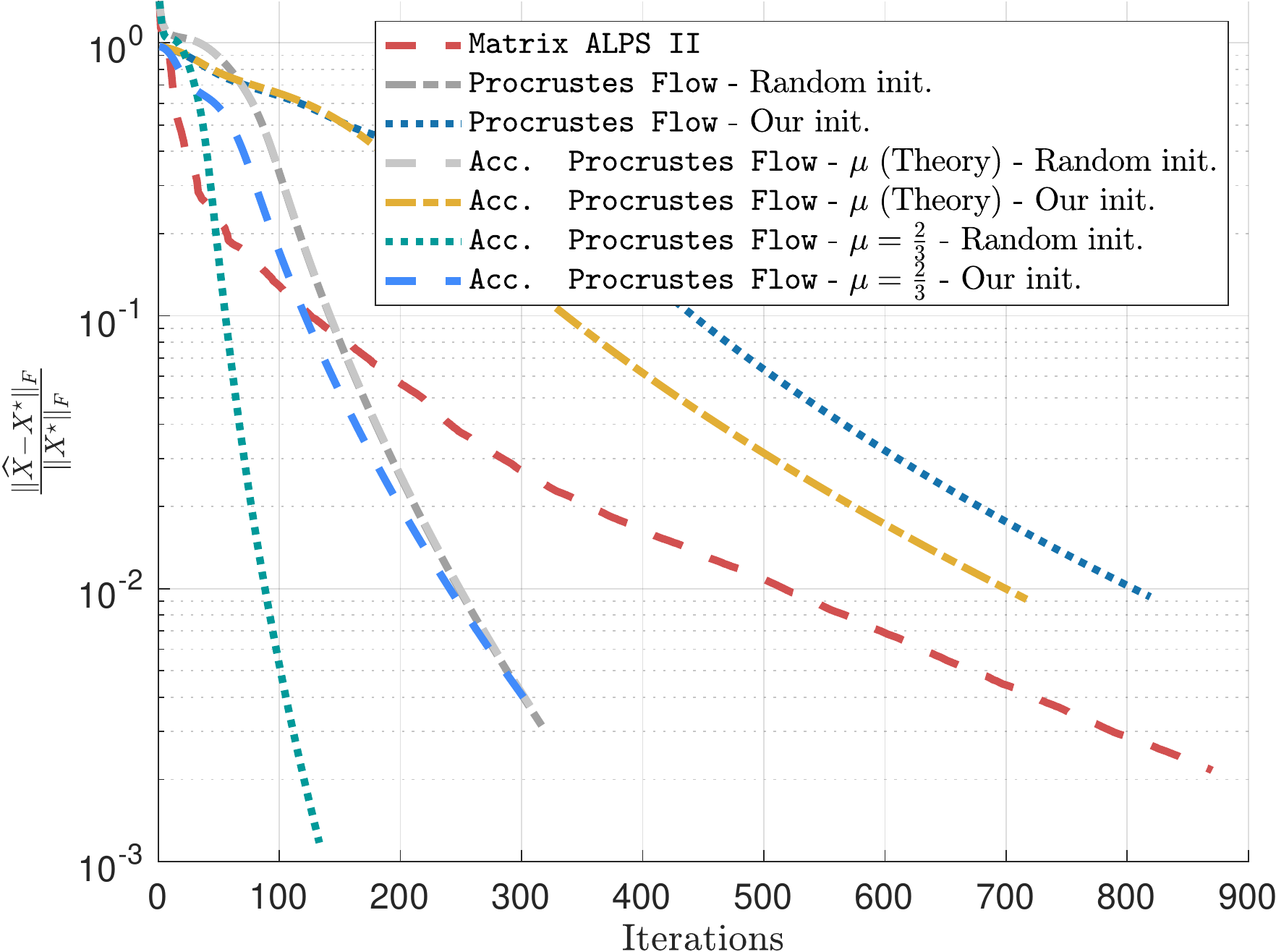} 
 \caption{Synthetic example results on low-rank matrix sensing. \emph{Top row}: Convergence behavior vs. time elapsed. \emph{Bottom row:} Convergence behavior vs. number of iterations. \emph{Left panel}: $C = 5$, noiseless case; \emph{Center panel:} $C = 3$, noiseless case; \emph{Right panel:} $C = 5$, noisy case, $\|w\|_2 = 0.01$.}
 \label{fig:00}
\end{figure*}

For the implementation of the Accelerated Procrustes Flow, we have used the momentum paramter $\mu$ given in Theorem \ref{thm:00},
wiith $\varepsilon = 1$, and a more ``aggressive'' selection, $\mu = \tfrac{2}{3}$. 
The latter value is larger than what our theory dictates, but as we have conjectured, our theory holds for different configurations of $(\mu, \delta_{2r})$; proving our theory for a less strict feasible values of these parameters remains an open problem.
Moreover, we have observed that various values of $\mu$ still lead to acceleration; in our experiments, we selected $\mu$ using grid search over the set $\{ \tfrac{1}{4}, \tfrac{1}{2}, \tfrac{2}{3}, \tfrac{3}{4} \}$ (details omitted).

The observations $y$ are set to
$y = \mathcal{A}\left(X^\star\right) + w$ for some noise vector $w$;
while the theory holds for the noiseless case, we show empirically that noisy cases are nicely handled by the same algorithm.
We use permuted and subsampled noiselets for the linear operator $\mathcal{A}$ \cite{waters2011sparcs}. 
The optimal matrix $X^\star$ is generated as the multiplication of a tall matrix $U^\star \in \mathbb{R}^{n \times r}$ such that $X^\star = U^\star U^{\star \top}$, and $\|X^\star \|_F = 1$, without loss of generality. 
The entries of $U^\star$ are drawn i.i.d. from a Gaussian distribution with zero mean and unit variance.
In the noisy case, $w$ has the same dimensions with $y$, its entries are drawn from a zero mean Gaussian distribution with norm $\|w\|_2 = 0.01$.
The random initialization is defined as $U_0$ drawn i.i.d. from a Gaussian distribution with zero mean and unit variance; the specific initialization is computed as the rank-$r$ approximation of the gradient at the zero point, according to Lemma \ref{lem:init}.

The results are shown in Figure \ref{fig:00}.
Some notable remarks: 
$i)$ While factorization techniques might take more iterations to converge compared to non-factorized algorithms, the per iteration time complexity is much less, such that overall, factorized gradient descent converges more quickly in terms of total execution time.
$ii)$ Our proposed algorithm, \emph{even under the restrictive assumptions on acceleration parameter $\mu$}, performs better than the non-accelerated factored gradient descent algorithms, such as Procrustes Flow.
$iii)$ Our theory is conservative: using a much larger $\mu$ we obtain a faster convergence; the proof for less strict assumptions for $\mu$ is an interesting future direction.
In all cases, our findings illustrate the effectiveness of the proposed schemes on different problem configurations.

\subsection{Neuron activity recovery from  $\mu$ECoG}\label{sec:Neuro}
A major hurdle in understanding the human brain is inferring the activities 
of individual neurons from  meso-scales cortical surface electrical potentials,
recorded by  electrocorticography (ECoG).
Here, we present a  novel application for low rank matrix sensing to recover single-neuron 
activities from  a single $\mu$ECoG  electrode. We simulated recordings of stimulus evoked cortical
surface electrical potentials
using a 
single $\mu$ECoG electrode over a short time period; we observe a vector
of electrical potentials $y\in\RR^m$ for $m$ time instances.
The membrane potentials of a set of 
neurons over this time period which we wish to recover can be 
viewed as a matrix $X\in\RR^{n\times m}$, where each row is the 
membrane potential for a neuron at $m$ time instances.
This matrix will be low rank since neurons are excited for a very short period, 
after the onset of the stimuli.

Based on~\cite{bouchard2018structure}, the activity of neurons are low-pass filtered with a cut-off frequency that depends on the distance of the neuron from the surface [$f_c(d)$], amplitude attenuated according to
distance [$A(d)$], producing distance dependent extracellular potentials.
The $\mu$ECoG recordings $y$ are the summation of these extracellular potentials. 
The distance dependent amplitude attenuation and  cut-off frequency are given by: $A(d)=\frac{1}{d^\alpha}\text{ and } f_c(d)=\frac{f_{\max}}{d^\alpha}$,
 $\alpha= \Delta_1,$ for $d<h$ and $\alpha= \Delta_2$ for $d\geq h$. Here $h$ is in units of distance from the brain surface to allow for potential piecewise linear.
 
We can view the $\mu$ECoG recordings $y$ as linear mapping of the membrane potentials of the neurons:
$y=\mathcal{A}(X)=\mbf{A}\tvec(X)$, where $\mbf{A}\in\RR^{m\times nm}$ is a banded matrix
which accounts for the distance dependent lowpass filtering and amplitude attenuation 
of the membrane potentials of neurons.
Experiments can lead to hours of  $\mu$ECoG recordings ($m$ in millions) and 
neural activities are defined for $n\approx31000$ neurons in a neocortical volume,
see; \cite{markram2015reconstruction}. Hence, the use of factorized algorithms, and their accelerations,
in these applications is necessary.

 \begin{figure}[tb!]
 \begin{center}
\includegraphics[width=1\columnwidth,,trim={3.0cm 0cm  3.0cm 0cm},clip]{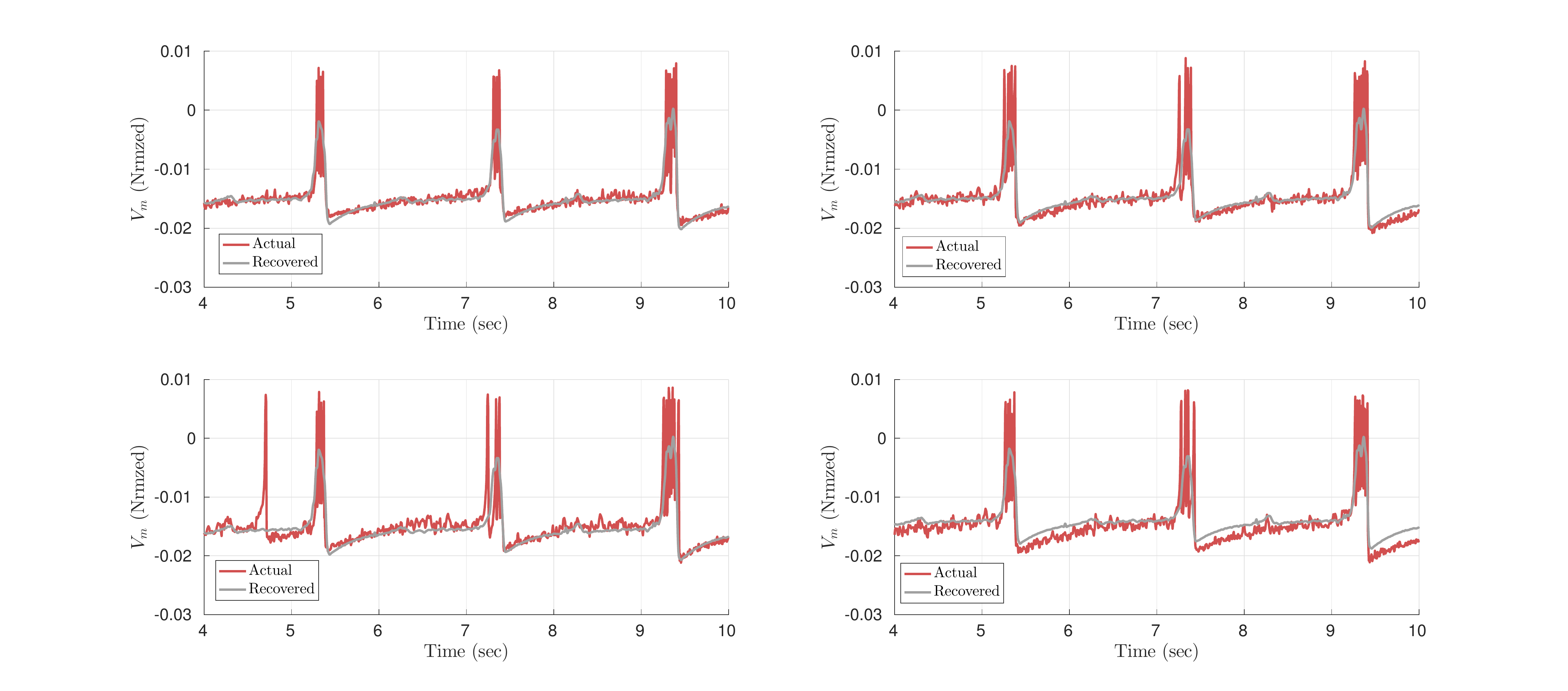}
\vskip -0.2in
 \caption{Neuronal activity recovery: The actual membrane potentials between 4 to 10 secs of simulation
 and the recovered potentials for four neurons at different distances.
 }\label{fig:Neuro1}
 \end{center}
\end{figure}

 
Here, we use
the low-rank matrix sensing model and a rectangular version of accelerated Procrustes flow (\emph{i.e.}, $X$ as $UV^\top$) to recover the 
neuronal activity $X$; we defer the reader to \cite[Section 3.1]{park2016finding} for an equivalent transformation between rectangular and square matrix factorizations.
In Figure~\ref{fig:Neuro1}, we present results for recovering neuronal potentials from
a $20secs$ ($m=4200$)  $\mu$ECoG simulation $y$; see also the appendix.
We note that our model recovers the potentials of individual neurons very well,
though not at single-action potential resolution.
\emph{To the best of our knowledge, such recovery results have not been demonstrated before}: they open the possibility of recovering the activities of individual neurons (the `microscopic units') from meso-scale signals recorded in humans (ECoG).

\section{Quantum state tomography (QST)}
We focus on  QST of a low-rank $q$-qubit state, $X^\star$, from measuring expectation values of $q$-qubit Pauli observables $\{A_i\}_{i=1}^m$, where $A_i = \otimes_{j=1}^q s_j$ and $\otimes$ denotes the Kronecker product.
Each $s_j$ is a $2 \times 2$ matrix from the set:
\begin{small}
\begin{align*}
\sigma_I = \begin{bmatrix} 
1 & 0 \\
0 & 1
\end{bmatrix}, \;
\sigma_x = \begin{bmatrix} 
0 & 1 \\
1 & 0
\end{bmatrix}, \;
\sigma_y = \begin{bmatrix} 
0 & -j \\
j & 0
\end{bmatrix}, \;
\sigma_z = \begin{bmatrix} 
1 & 0 \\
0 & -1
\end{bmatrix}, 
\end{align*}
\end{small}
where $j$, here, denotes the imaginary number.
$y \in \mathbb{R}^m$ is the measurement vector with elements $y_i = \text{Tr}(A_i \cdot X^\star)+e_i,~i = 1, \dots, m$, for some error $e_i$. 
Pauli measurements satisfy RIP, as follows:
\begin{lemma}[RIP for Pauli measurements \cite{liu2011universal}]\label{def:RIP1}
Let $\mathcal{A} : \mathbb{C}^{2^q \times 2^q} \rightarrow \mathbb{R}^m$ be such that \begin{small}$(\mathcal{A}(X^\star))_i = \tfrac{2^q}{\sqrt{m}} \emph{\text{Tr}}(A_i \cdot X^\star)$\end{small}, where $A_i$ are Pauli operators.
Then, with high probability over the choice of \begin{small}$m = \tfrac{c}{\delta_r^2}\cdot (r 2^q q^6)$\end{small} $A_i$'s, where $c > 0$ is a constant, $\mathcal{A}$ satisfies the $r$-RIP with constant $\delta_r$, $ 0 \leq \delta_r < 1$. 
\end{lemma}

\begin{figure}[h]
  \begin{center}
    \includegraphics[width=0.48\columnwidth]{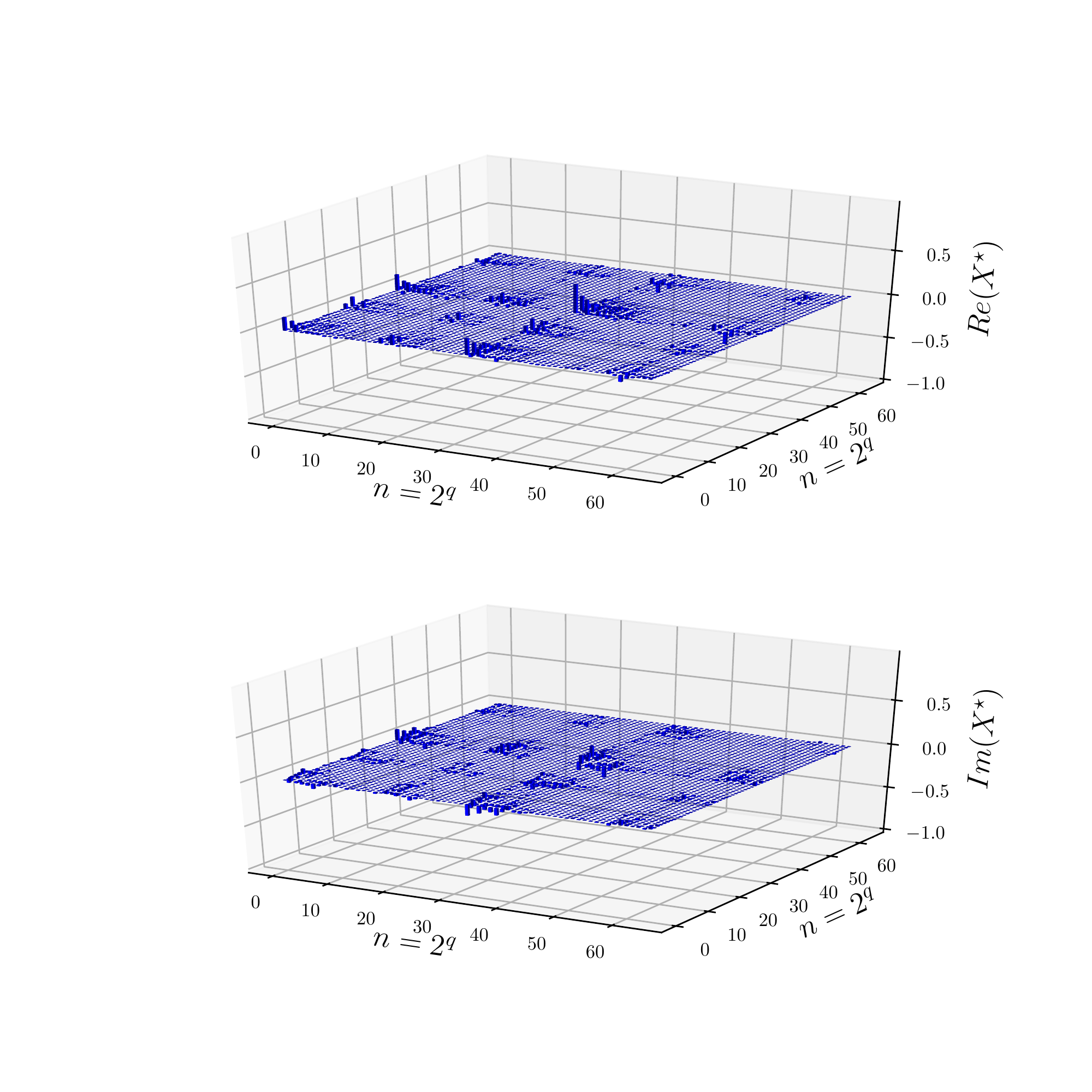}
    \includegraphics[width=0.48\columnwidth]{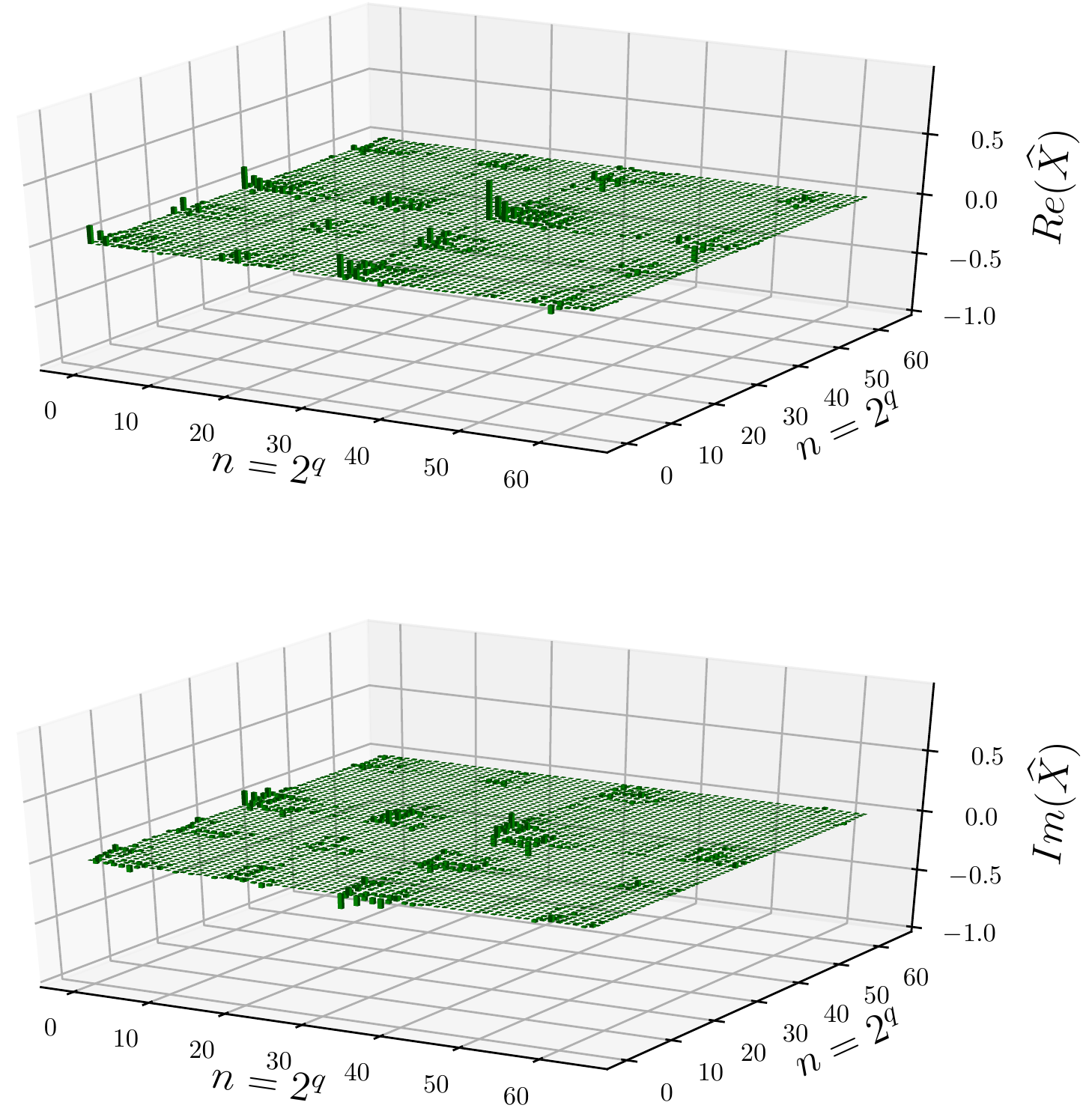}
    \vskip -0.1in
    \caption{Real and imaginary parts of $X^\star$ (Left) and its reconstruction $\widehat{X}$ for $\mu=\sfrac{3}{4}$.}
      \label{fig:qst_random}
  \end{center}
\end{figure}

\begin{figure*}[!h]
  \begin{center}
    \includegraphics[width=0.33\textwidth]{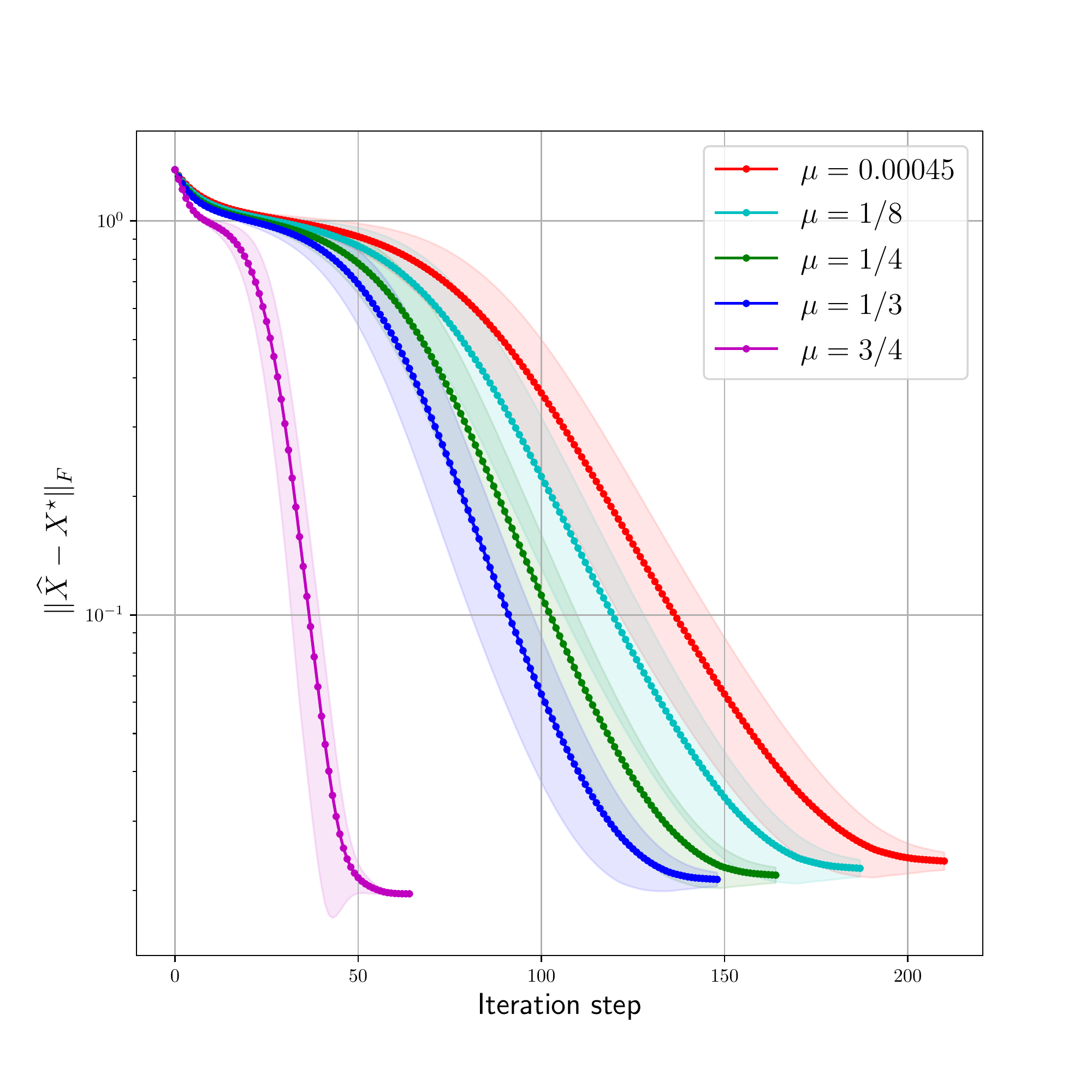}  
        \includegraphics[width=0.33\textwidth]{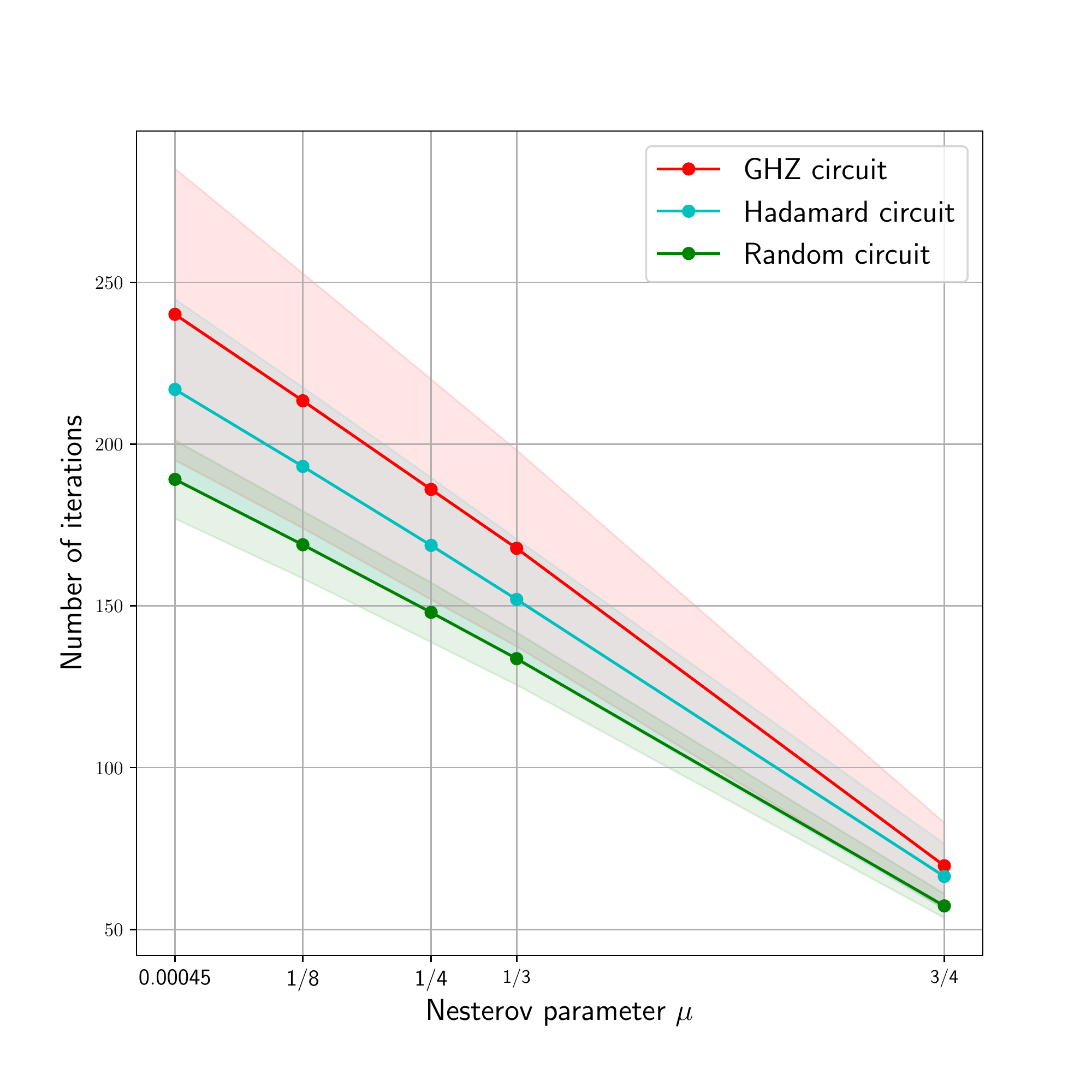}
    \includegraphics[width=0.33\textwidth]{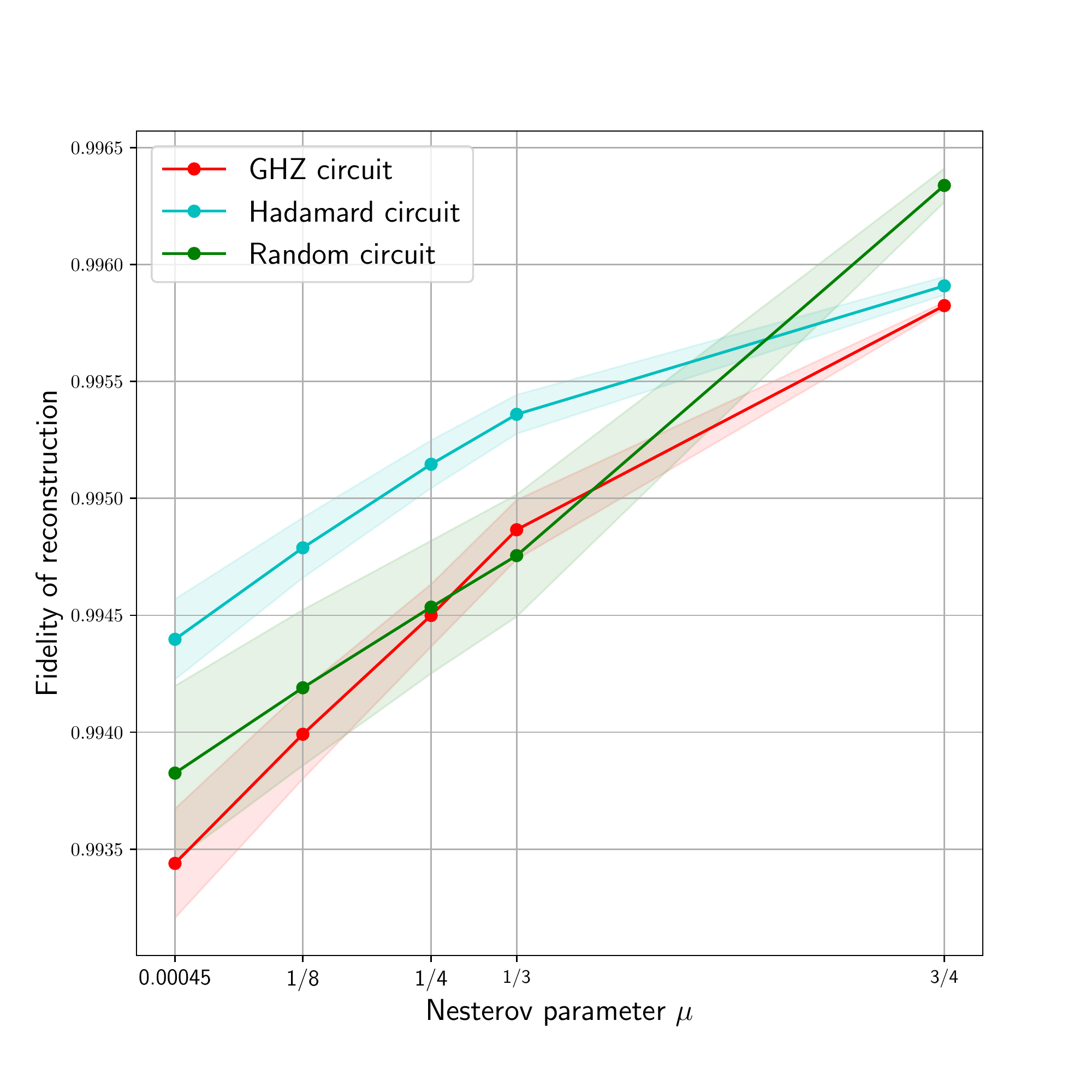}
    \vskip -0.1in
    \caption{
      \emph{Left panel:} $\| \widehat{X} - X^\star \|_F$ vs. \# of iterations for different $\mu$;
      \emph{Middle panel:} \# of iterations to reach $\texttt{reltol}$ vs. $\mu$ and for different circuits $(X^\star)$;
      \emph{Right panel:} Fidelity of $\widehat{X}$, defined as $\text{Tr}\big(\sqrt{ \sqrt{X^\star} \widehat{X} \sqrt{X^\star}}\big)^2$, vs. $\mu$ parameters and for different circuits $(X^\star)$. Shaded area denotes standard deviation around the mean over repeated runs in all cases.}
    \label{fig:qst_error_convergence_fidelity}
  \end{center}
\end{figure*}

We report on the reconstruction of the density matrix $X^\star$ of quantum circuits consisting of $q=6$ qubits; \emph{i.e.}, $n = 2^q = 64$.
Higher dimension experiments are left for future work.
We design $X^\star$ as a pure, rank-1 density matrix, such that $X \succeq 0$.
It is synthesized by the application of \texttt{CNOT} gates on pairs of qubits (GHZ circuit), of \texttt{Hadamard} gates on all qubits (Hadamard circuit) and of a random selection of \texttt{CNOT} and \texttt{U3} gates on pairs of and individual qubits, respectively, organized in a structure that is $40$ levels deep (Random circuit). 
We use IBM Quantum Information Software Kit (Qiskit) \cite{qiskit} for preparing the circuits and presenting them to QASM Simulator, which is a high performance quantum circuit simulator written in C++ that includes a variety of realistic circuit level noise models. 

We randomly chose informationally incomplete subsets of the measurements to generate $y$, where $m = \lceil{0.6 \cdot n^2} \rceil$; for each measurement we simulated $8192$ pulses and converted the measurements to single-number correlation measurements. 
We test the behavior of our algorithm for different values of the momentum parameter: \begin{small}$\mu = \frac{\sigma_r(X^\star)^{1/2}}{10^3 \sqrt{\kappa\tau(X^\star)}} \cdot \frac{\varepsilon}{2 \cdot \sigma_1(X^\star)^{1/2} \cdot r}$\end{small}, and constant $\mu \in \{\frac{1}{8}, \frac{1}{4}, \frac{1}{3}, \frac{3}{4}\}$.
Observe that for this application, $\sigma_r(X^\star) = 1$, $\tau(X^\star) = 1$, and $r=1$ by construction; we also approximated $\kappa = 1.223$, which, for $\varepsilon = 1$, results in $\mu = 4.5 \cdot 10^{-4}$.
We run each QST experiment for $10$ times for random initializations. 
We record the evolution of error at each step and the fidelity of the reconstruction, and stop when the relative error of successive iterates gets smaller than $\texttt{reltol} = 5\times 10^{-4}$ or the number of iterations exceeds $1000$ - whichever happens first.



Figure~\ref{fig:qst_random} shows the complex representation of $i)$ the target $X^{\star}$ (Left panel), and $ii)$ its reconstruction $\widehat{X}$ (Right panel). 
We observe that $X^\star$ and $\widehat{X}$ are indistinguishable: our scheme succeeds in reproducing the finer details in the density matrix structure. 
Figure \ref{fig:qst_error_convergence_fidelity} summarizes the performance of our proposal on different $X^\star$ (named circuits in quantum information), and for different $\mu$ values.
Left panel shows the evolution across iterations, featuring a steep dive to convergence for the largest value of $\mu$ we tested: we report that we also tested $\mu = 0$, which shows only slight worse performances than $\mu = 4.5 \cdot 10^{-4}$, and it is omitted.
Center panel summarizes the acceleration benefits of $\mu$ parameter:
the \# of iterations to reach $\texttt{reltol} = 5\times 10^{-4}$ decreases roughly $\times 3$, as we increase $\mu$; this behavior is consistent across all our test quantum circuits. 
On the right, larger $\mu$'s not only accelerate, but also slightly improve the reconstruction quality. 

 \section{Conclusion}
 We have introduced the accelerated Procrustes Flow algorithm for the factorized form of the low-rank matrix sensing problems. 
 We proved that under certain assumptions on the problem parameters, the accelerated Procrustes Flow algorithm converges linearly to a neighborhood of the optimal solution, whose size depends on the momentum parameter $\mu$.
 We demonstrate empirically using both simulated and real data that our algorithm outperforms non-accelerated methods on both the original problem domain and the factorized space.
 These results expand on existing work in the literature illustrating the promise of factorized methods for certain low-rank matrix problems.

\begin{quote}
\begin{small}
\bibliographystyle{aaai}
\bibliography{biblio}
\end{small}
\end{quote}

\clearpage
\onecolumn

\section{Supporting lemmata}

In this section, we present a series of lemmata, used for the main result of the paper.

\begin{lemma}{\label{lem:supp_00}}
Let $U \in \mathbb{R}^{n \times r}$ and $U^\star \in \mathbb{R}^{n \times r}$, such that $\|U - U^\star R\|_F \leq \tfrac{\sigma_r(X^\star)^{1/2}}{10^3 \sqrt{\kappa \tau(X^\star)}}$ for some $R \in \mathcal{O}$, where $X^\star = U^\star U^{\star \top}$, $\kappa := \tfrac{1 + \delta_{2r}}{1 - \delta_{2r}} > 1$, for $\delta_{2r} \leq \tfrac{1}{10}$, and $\tau(X^\star) := \tfrac{\sigma_1(X^\star)}{\sigma_r(X^\star)} > 1$. Then:
\begin{align*}
\sigma_1(X^\star)^{1/2} \left(1  - \tfrac{1}{10^3}\right) &\leq \sigma_1(U) \leq \sigma_1(X^\star)^{1/2} \left(1 + \tfrac{1}{10^3} \right) \\
\sigma_r(X^\star)^{1/2} \left(1  - \tfrac{1}{10^3}\right) &\leq \sigma_r(U) \leq \sigma_r(X^\star)^{1/2} \left(1 + \tfrac{1}{10^3} \right)
\end{align*}
\end{lemma}

\begin{proof}
By the fact $\|\cdot\|_2 \leq \|\cdot\|_F$ and using Weyl's inequality for perturbation of singular values \cite[Theorem 3.3.16]{horn1990matrix}, we have:
\begin{align*}
\left|\sigma_i(U) - \sigma_i(U^\star)\right| \leq \tfrac{\sigma_r(X^\star)^{1/2}}{10^3 \sqrt{\kappa \tau(X^\star)}} \leq \tfrac{\sigma_r(X^\star)^{1/2}}{10^3}, \quad 1 \leq i \leq r.
\end{align*}
Then, 
\begin{align*}
- \tfrac{\sigma_r(X^\star)^{1/2}}{10^3} &\leq \sigma_1(U) - \sigma_1(U^\star) \leq \tfrac{\sigma_r(X^\star)^{1/2}}{10^3} \Rightarrow \\
\sigma_1(X^\star)^{1/2} - \tfrac{\sigma_r(X^\star)^{1/2}}{10^3} &\leq \sigma_1(U) \leq \sigma_1(X^\star)^{1/2} + \tfrac{\sigma_r(X^\star)^{1/2}}{10^3} \Rightarrow \\
\sigma_1(X^\star)^{1/2} \left(1  - \tfrac{1}{10^3}\right) &\leq \sigma_1(U) \leq \sigma_1(X^\star)^{1/2} \left(1 + \tfrac{1}{10^3} \right).
\end{align*}
Similarly:
\begin{align*}
- \tfrac{\sigma_r(X^\star)^{1/2}}{10^3} &\leq \sigma_r(U) - \sigma_r(U^\star) \leq \tfrac{\sigma_r(X^\star)^{1/2}}{10^3} \Rightarrow \\
\sigma_r(X^\star)^{1/2} - \tfrac{\sigma_r(X^\star)^{1/2}}{10^3} &\leq \sigma_r(U) \leq \sigma_r(X^\star)^{1/2} + \tfrac{\sigma_r(X^\star)^{1/2}}{10^3} \Rightarrow \\
\sigma_r(X^\star)^{1/2} \left(1  - \tfrac{1}{10^3}\right) &\leq \sigma_r(U) \leq \sigma_r(X^\star)^{1/2} \left(1 + \tfrac{1}{10^3} \right).
\end{align*}
In the above, we used the fact that $\sigma_i(U^\star) = \sigma_i(X^\star)^{1/2}$, for all $i$, and the fact that $\sigma_i(X^\star)^{1/2} \geq \sigma_j(X^\star)^{1/2}$, for $i \leq j$.
\end{proof}

\begin{lemma}{\label{lem:supp_02}}
Let $U \in \mathbb{R}^{n \times r}, U_{-} \in \mathbb{R}^{n \times r}$, and $U^\star \in \mathbb{R}^{n \times r}$, such that 
\begin{small}$\min_{R \in \mathcal{O}} \|U - U^\star R\|_F \leq \tfrac{\sigma_r(X^\star)^{1/2}}{10^3 \sqrt{\kappa\tau(X^\star)}}$\end{small} ~and~
\begin{small}$\min_{R \in \mathcal{O}} \|U_- - U^\star R\|_F \leq \tfrac{\sigma_r(X^\star)^{1/2}}{10^3 \sqrt{\kappa\tau(X^\star)}}$\end{small}, where $X^\star = U^\star U^{\star \top}$, and \begin{small}$\kappa := \tfrac{1 + \delta_{2r}}{1 - \delta_{2r}} > 1$\end{small}, for \begin{small}$\delta_{2r} \leq \tfrac{1}{10}$\end{small}, and \begin{small}$\tau(X^\star) := \tfrac{\sigma_1(X^\star)}{\sigma_r(X^\star)} > 1$\end{small}.
Set the momentum parameter as \begin{small}$\mu = \frac{\sigma_r(X^\star)^{1/2}}{10^3 \sqrt{\kappa\tau(X^\star)}} \cdot \frac{\varepsilon}{2 \cdot \sigma_1(X^\star)^{1/2} \cdot r}$\end{small}, for $\varepsilon \in (0, 1)$ user-defined.
Then, 
\begin{align*}
\|Z - U^\star R_Z\|_F \leq \left(\tfrac{3}{2} + 2|\mu|\right) \cdot \tfrac{\sigma_r(X^\star)^{1/2}}{10^3\sqrt{\kappa \tau(X^\star)}}.
\end{align*}
\end{lemma}

\begin{proof}
Let $R_{U} \in \arg\min_{R \in \mathcal{O}} \|U - U^{*}\|_{F}$ and $R_{U_{-}} \in \arg\min_{R \in \mathcal{O}} \|U_{-} - U^{*}R\|_{F}$.
By the definition of the distance function:
\begin{align*}
\|Z - U^\star R_Z\|_F &= \min_{R \in \mathcal{O}} \|Z - U^\star R\|_F = \min_{R \in \mathcal{O}} \| U + \mu (U - U_{-}) - U^\star R \|_F \\
				     &= \min_{R \in \mathcal{O}} \|U + \mu (U - U_{-}) - (1 - \mu + \mu) U^\star R\|_F \\
				     &\leq |1 + \mu|\cdot\|U - U^{*}R_{U}\|_{F} + |\mu|\cdot||U_{-} - U^{*}R_{U_{-}}||_{F}\\
				     &= |1 + \mu|\cdot\|U - U^{*}R_{U}\|_{F} + |\mu|\cdot\|U_{-} - U^{*}R_{U} - U^{*}R_{U_{-}} + U^{*}R_{U_{-}}\|_{F}\\
				     &= |1 + \mu|\cdot\|U - U^{*}R_{U}\|_{F} + |\mu|\cdot\|(U_{-} - U^{*}R_{U_{-}}) + U^{*}(R_{U_{-}} - R_{U})\|_{F}\\
				     &\leq |1 + \mu| \cdot \min_{R \in \mathcal{O}} \|U - U^\star R\|_F + |\mu| \cdot \min_{R \in \mathcal{O}} \|U_- - U^\star R\|_F + |\mu| \cdot \|U^\star (R_U - R_{U_-})\|_F \\
				     &\leq |1 + \mu| \cdot \min_{R \in \mathcal{O}} \|U - U^\star R\|_F + |\mu| \cdot \min_{R \in \mathcal{O}} \|U_- - U^\star R\|_F + 2|\mu| \cdot \sigma_1(X^\star)^{1/2} r \\
				     &\stackrel{(i)}{\leq} \left(\tfrac{3}{2} + 2|\mu|\right) \cdot \tfrac{\sigma_r(X^\star)^{1/2}}{10^3\sqrt{\kappa \tau(X^\star)}}
\end{align*}
where $(i)$ is due to the fact that $\mu \leq \frac{\sigma_r(X^\star)^{1/2}}{10^3 \sqrt{\kappa\tau(X^\star)}} \cdot \frac{1}{2 \cdot \sigma_1(X^\star)^{1/2} \cdot r}$.
We keep $\mu$ in the expression, but we use it for clarity for the rest of the proof.
\end{proof}

\begin{corollary}{\label{cor:2}}
Let $Z \in \mathbb{R}^{n \times r}$ and $U^\star \in \mathbb{R}^{n \times r}$, such that 
$\|Z - U^\star R\|_F \leq \left(\frac{3}{2} + 2|\mu|\right) \cdot \tfrac{\sigma_r(X^\star)^{1/2}}{10^3\sqrt{\kappa \tau(X^\star)}}$ for some $R \in \mathcal{O}$, and $X^\star = U^\star U^{\star \top}$. Then:
\begin{align*}
\sigma_1(X^\star)^{1/2} \left(1  - \left(\tfrac{3}{2} + 2|\mu|\right)\tfrac{1}{10^3}\right) &\leq \sigma_1(Z) \leq \sigma_1(X^\star)^{1/2} \left(1 + \left(\tfrac{3}{2} + 2|\mu|\right)\tfrac{1}{10^3} \right) \\
\sigma_r(X^\star)^{1/2} \left(1  - \left(\tfrac{3}{2} + 2|\mu|\right)\tfrac{1}{10^3}\right) &\leq \sigma_r(Z) \leq \sigma_r(X^\star)^{1/2} \left(1 + \left(\tfrac{3}{2} + 2|\mu|\right)\tfrac{1}{10^3} \right).
\end{align*}
Given that  \begin{small}$\mu = \frac{\sigma_r(X^\star)^{1/2}}{10^3 \sqrt{\kappa\tau(X^\star)}} \cdot \frac{\varepsilon}{2 \cdot \sigma_1(X^\star)^{1/2} \cdot r} \leq \tfrac{1}{10^3}$\end{small}, we get:
\begin{align*}
0.998\cdot \sigma_1(X^\star)^{1/2} &\leq \sigma_1(Z) \leq 1.0015 \cdot \sigma_1(X^\star)^{1/2} \\
0.998\cdot \sigma_r(X^\star)^{1/2}  &\leq \sigma_r(Z) \leq 1.0015 \cdot \sigma_r(X^\star)^{1/2}.
\end{align*}
\end{corollary}

\begin{proof}
The proof follows similar motions as in Lemma \ref{lem:supp_00}.
\end{proof}

\begin{corollary}{\label{cor:01}}
Under the same assumptions of Lemma \ref{lem:supp_00} and Corollary \ref{cor:2}, and given the assumptions on $\mu$, we have:
\begin{align*}
\tfrac{99}{100} \cdot \|X^\star\|_2 \leq \|ZZ^\top\|_2 \leq \tfrac{101}{100} \cdot \|X^\star\|_2 \\
\tfrac{99}{100} \cdot \|X^\star\|_2 \leq \|Z_0Z_0^\top\|_2 \leq \tfrac{101}{100} \cdot \|X^\star\|_2 
\end{align*}
and 
\begin{align*}
\tfrac{99}{101} \cdot \|Z_0Z_0^\top\|_2 \leq \|ZZ^\top\|_2 \leq \tfrac{101}{99} \cdot \|Z_0Z_0^\top\|_2
\end{align*}
\end{corollary}
\begin{proof}
The proof is easily derived based on the quantities from Lemma \ref{lem:supp_00} and Corollary \ref{cor:2}.
\end{proof}

\begin{corollary}{\label{cor:1}}
Let $Z \in \mathbb{R}^{n \times r}$ and $U^\star \in \mathbb{R}^{n \times r}$, such that 
$\|Z - U^\star R\|_F \leq \left(\tfrac{3}{2} + 2|\mu|\right) \cdot \tfrac{\sigma_r(X^\star)^{1/2}}{10^3 \sqrt{\kappa\tau(X^\star)}}$ for some $R \in \mathcal{O}$, and $X^\star = U^\star U^{\star \top}$.
Define $\tau(W) = \frac{\sigma_1(W)}{\sigma_r(W)}$. 
Then:
\begin{align*}
\tau(ZZ^\top) \leq \beta^2 \tau(X^\star),
\end{align*}
where $\beta := \frac{1 + \left(\tfrac{3}{2} + 2|\mu|\right) \cdot \tfrac{1}{10^3}}{1 - \left(\tfrac{3}{2} + 2|\mu| \right) \cdot \tfrac{1}{10^3}} > 1$.
for $\mu \leq \frac{\sigma_r(X^\star)^{1/2}}{10^3 \sqrt{\kappa\tau(X^\star)}} \cdot \frac{1}{2 \cdot \sigma_1(X^\star)^{1/2} \cdot r}$.
\end{corollary}

\begin{proof}
The proof uses the definition of the condition number $\tau(\cdot)$ and the results from Lemma \ref{lem:supp_00} and and Corollary \ref{cor:2}.
\end{proof}

\begin{lemma}{\label{lem:equiveta}}
Consider the following three step sizes:
\begin{align*}
\eta &= \frac{1}{4\left( (1 + \delta_{2r}) \|Z_0Z_0^\top\|_2 + \|\mathcal{A}^\dagger \left(\mathcal{A}( Z_0Z_0^\top ) - y\right)\|_2 \right) } \\
\widehat{\eta} &= \frac{1}{4\left( (1 + \delta_{2r}) \|ZZ^\top\|_2 + \|\mathcal{A}^\dagger \left(\mathcal{A}( ZZ^\top ) - y\right) Q_ZQ_Z^\top \|_2 \right) } \\
\eta^\star &= \frac{1}{4\left( (1 + \delta_{2r}) \|X^{\star \top}\|_2 + \|\mathcal{A}^\dagger \left(\mathcal{A}( X^{\star} ) - y\right) \|_2 \right) }.
\end{align*}
Here, $Z_0 \in \mathbb{R}^{n \times r}$ is the initial point, $Z \in \mathbb{R}^{n \times r}$ is the current point,$X^\star \in \mathbb{R}^{n \times n}$ is the optimal solution, and $Q_Z$ denotes a basis of the column space of $Z$. 
Then, under the assumptions that $\min_{R \in \mathcal{O}} \|U - U^\star R\|_F \leq \tfrac{\sigma_r(X^\star)^{1/2}}{10^3 \sqrt{\kappa\tau(X^\star)}}$, 
and $\min_{R \in \mathcal{O}} \|Z - U^\star R\|_F \leq \left( \tfrac{3}{2} + 2|\mu|\right) \cdot \tfrac{\sigma_r(X^\star)^{1/2}}{10^3 \sqrt{\kappa \tau(X^\star)}}$, and assuming \begin{small}$\mu = \frac{\sigma_r(X^\star)^{1/2}}{10^3 \sqrt{\kappa\tau(X^\star)}} \cdot \frac{\varepsilon}{2 \cdot \sigma_1(X^\star)^{1/2} \cdot r}$\end{small}, for the user-defined parameter $\varepsilon \in (0, 1)$, we have:
\begin{align*}
\tfrac{10}{9} \eta \geq \widehat{\eta} \geq \tfrac{10}{10.5} \eta, \quad \text{and} \quad \tfrac{100}{102} \eta^\star \leq \eta \leq \tfrac{102}{100} \eta^\star
\end{align*}
\end{lemma}

\begin{proof}
The assumptions of the lemma are identical to that of Corollary \ref{cor:01}.
Thus, we have:
$\tfrac{99}{100} \cdot \|U^\star\|_2^2 \leq \|Z\|_2^2 \leq \tfrac{101}{100} \cdot \|U^\star\|_2^2$, 
$\tfrac{99}{100} \cdot \|U^\star\|_2^2 \leq \|Z_0\|_2^2 \leq \tfrac{101}{100} \cdot \|U^\star\|_2^2$,
and 
$\tfrac{99}{101} \cdot \|Z_0\|_2^2 \leq \|Z\|_2^2 \leq \tfrac{101}{99} \cdot \|Z_0\|_2^2.$
We focus on the inequality $\widehat{\eta} \geq \tfrac{10}{10.5} \eta$. 
Observe that:
\begin{small}
\begin{align*}
\left\|\mathcal{A}^\dagger \left(\mathcal{A}(ZZ^\top) - y \right)Q_ZQ_Z^\top \right\|_2 &\leq \left\|\mathcal{A}^\dagger \left(\mathcal{A}(ZZ^\top) - y \right) \right\|_2 \\
&= \left\|\mathcal{A}^\dagger \left(\mathcal{A}(ZZ^\top) - y \right) - \mathcal{A}^\dagger \left(\mathcal{A}(Z_0Z_0^\top) - y \right) + \mathcal{A}^\dagger \left(\mathcal{A}(Z_0Z_0^\top) - y \right)\right\|_2 \\
&\stackrel{(i)}{\leq} (1 + \delta_{2r}) \left\|ZZ^\top - Z_0 Z_0^\top \right\|_F + \left \|\mathcal{A}^\dagger \left(\mathcal{A}(Z_0Z_0^\top) - y \right)\right\|_2 \\
&\leq (1 + \delta_{2r}) \left\|ZZ^\top - U^\star U^{\star \top} \right\|_F + (1 + \delta_{2r}) \left\|Z_0 Z_0^\top - U^\star U^{\star \top} \right\|_F \\
&\quad \quad \quad \quad + \left \|\mathcal{A}^\dagger \left(\mathcal{A}(Z_0Z_0^\top) - y \right)\right\|_2
\end{align*}
\end{small}
where $(i)$ is due to smoothness via RIP constants of the objective and the fact $\|\cdot\|_2 \leq \|\cdot \|_F$.
For the first two terms on the right-hand side, where $R_Z$ is the minimizing rotation matrix for $Z$, we obtain:
\begin{align*}
\|ZZ^\top - U^\star U^{\star \top}\|_F &= \|ZZ^\top - U^\star R_Z Z^\top + U^\star R_Z Z^\top - U^\star U^{\star \top}\|_F \\
&= \|(Z -U^\star R_Z) Z^\top + U^{\star} R_Z ( Z - U^\star R_Z)^\top\|_F \\
&\leq \|Z\|_2 \cdot \|Z - U^\star R_Z\|_F  + \|U^\star \|_2 \cdot \|Z - U^\star R_Z\|_F \\
&\leq \left(\|Z\|_2 + \|U^\star\|_2 \right) \cdot \|Z - U^\star R_Z\|_F \\ 
&\stackrel{(i)}{\leq} \left( \sqrt{\tfrac{101}{99}} + \sqrt{\tfrac{100}{99}}\right) \|Z_0\|_2 \cdot \|Z - U^\star R_Z\|_F \\
&\stackrel{(ii)}{\leq}  \left( \sqrt{\tfrac{101}{99}} + \sqrt{\tfrac{100}{99}}\right) \|Z_0\|_2 \cdot 0.001 \sigma_r(X^\star)^{1/2} \\
&\leq \left( \sqrt{\tfrac{101}{99}} + \sqrt{\tfrac{100}{99}}\right) \cdot 0.001 \cdot \sqrt{\tfrac{100}{99}} \cdot \|Z_0\|_2^2
\end{align*}
where $(i)$ is due to the relation of $\|Z\|_2$ and $\|U^\star\|_2$ derived above,
$(ii)$ is due to Lemma \ref{lem:supp_02}.
Similarly:
\begin{align*}
\|Z_0Z_0^\top - U^\star U^{\star \top}\|_F \leq \left( \sqrt{\tfrac{101}{99}} + \sqrt{\tfrac{100}{99}}\right) \cdot 0.001 \cdot \sqrt{\tfrac{100}{99}} \cdot \|Z_0\|_2^2
\end{align*}
Using these above, we obtain:
\begin{small}
\begin{align*}
\left\|\mathcal{A}^\dagger \left(\mathcal{A}(ZZ^\top) - y \right)Q_ZQ_Z^\top \right\|_2 &\leq \tfrac{4.1 (1 + \delta_{2r})}{10^3} \|Z_0Z_0^\top\|_2 + \left \|\mathcal{A}^\dagger \left(\mathcal{A}(Z_0Z_0^\top) - y \right)\right\|_2
\end{align*}
\end{small}
Thus:
\begin{align*}
\widehat{\eta} &= \frac{1}{4\left( (1 + \delta_{2r}) \|ZZ^\top\|_2 + \|\mathcal{A}^\dagger \left(\mathcal{A}( ZZ^\top ) - y\right) Q_ZQ_Z^\top \|_2 \right) } \\
		        &\geq \frac{1}{4\left( (1 + \delta_{2r}) \tfrac{101}{99} \|Z_0Z_0\|_2 + \right) + \tfrac{4.1 (1 + \delta_{2r})}{10^3} \|Z_0Z_0^\top\|_2 + \left \|\mathcal{A}^\dagger \left(\mathcal{A}(Z_0Z_0^\top) - y \right)\right\|_2} \\ 
		        &\geq \frac{1}{4\left( \tfrac{10.5}{10} \cdot (1 + \delta_{2r})\|Z_0Z_0^\top\|_2 + \|\mathcal{A}^\dagger \left(\mathcal{A}( Z_0Z_0^\top ) - y\right) \|_2 \right) } \\ &\geq \tfrac{10}{10.5} \eta
\end{align*}
Similarly, one gets $\widehat{\eta} \leq \tfrac{10}{9} \eta$.

For the relation between $\eta$ and $\eta^\star$, we will prove here the lower bound; similar motions lead to the upper bound also.
By definition, and using the relations in Corollary \ref{cor:01}, we get:
\begin{align*}
\eta &= \frac{1}{4\left( (1 + \delta_{2r}) \|Z_0Z_0^\top\|_2 + \|\mathcal{A}^\dagger \left(\mathcal{A}( Z_0Z_0^\top ) - y\right)\|_2 \right) } \\
         &\geq \frac{1}{4\left( (1 + \delta_{2r}) \tfrac{101}{100} \|X^{\star \top}\|_2 + \|\mathcal{A}^\dagger \left(\mathcal{A}( Z_0Z_0^\top ) - y\right)\|_2 \right) }
\end{align*}
For the gradient term, we observe:
\begin{align*}
\left\|\mathcal{A}^\dagger \left(\mathcal{A}( Z_0Z_0^\top ) - y\right)\right\|_2 &\leq \left\|\mathcal{A}^\dagger \left(\mathcal{A}( Z_0Z_0^\top ) - y\right) -\mathcal{A}^\dagger \left(\mathcal{A}( X^\star ) - y\right) \right\|_2 + \left\|\mathcal{A}^\dagger \left(\mathcal{A}( X^\star ) - y\right)\right\|_2 \\
														     &\stackrel{(i)}{=} \left\|\mathcal{A}^\dagger \left(\mathcal{A}( Z_0Z_0^\top ) - y\right) -\mathcal{A}^\dagger \left(\mathcal{A}( X^\star ) - y\right) \right\|_2 \\
														     &\stackrel{(ii)}{\leq} (1 + \delta_{2r}) \left\|Z_0Z_0^\top - U^\star U^{\star \top} \right\|_F \\
														     &\stackrel{(iii)}{\leq} (1 + \delta_{2r}) \left( \|Z_0\|_2 + \|U^\star\|_2 \right) \cdot \|Z - U^\star R_Z\|_F \\
														     &\stackrel{(iv)}{\leq} (1 + \delta_{2r}) \left( \sqrt{\tfrac{101}{100}} + 1\right) \|U^\star\|_2 \cdot 0.001 \cdot \|U^\star\|_2^2 \\
														     &\leq 0.002 \cdot (1 + \delta_{2r}) \|X^\star\|_2
\end{align*}
where $(i)$ is due to $\left\|\mathcal{A}^\dagger \left(\mathcal{A}( X^\star ) - y\right)\right\|_2 = 0$,
$(ii)$ is due to the restricted smoothness assumption and the RIP, 
$(iii)$ is due to the bounds above on $\left\|Z_0Z_0^\top - U^\star U^{\star \top} \right\|_F$,
$(iv)$ is due to the bounds on $\|Z_0\|_2$, w.r.t. $\|U^\star\|_2$, as well as the bound on $\|Z - U^\star R\|_F$.

Thus, in the inequality above, we get:
\begin{align*}
\eta &\geq \frac{1}{4\left( (1 + \delta_{2r}) \tfrac{101}{100} \|X^{\star \top}\|_2 + \|\mathcal{A}^\dagger \left(\mathcal{A}( Z_0Z_0^\top ) - y\right)\|_2 \right) } \\
         &\geq \frac{1}{4\left( (1 + \delta_{2r}) \tfrac{101}{100} \|X^{\star \top}\|_2 + 0.001 \cdot (1 + \delta_{2r}) \|X^\star\|_2 + \|\mathcal{A}^\dagger \left(\mathcal{A}( X^\star ) - y\right)\|_2 \right) } \\
         &\geq \frac{1}{4\left( (1 + \delta_{2r}) \tfrac{102}{100} \|X^{\star \top}\|_2 + \|\mathcal{A}^\dagger \left(\mathcal{A}( X^\star ) - y\right)\|_2 \right) } \geq \tfrac{100}{102} \eta^\star
\end{align*}
Similarly, one can show that $\frac{102}{100}\eta^{\star} \geq \eta$. 
\end{proof}

\begin{lemma}{\label{lem:000}}
Let $U \in \mathbb{R}^{n \times r}, U_{-} \in \mathbb{R}^{n \times r}$, and $U^\star \in \mathbb{R}^{n \times r}$, such that 
\begin{small}$\min_{R \in \mathcal{O}} \|U - U^\star R\|_F \leq \tfrac{\sigma_r(X^\star)^{1/2}}{10^3 \sqrt{\kappa\tau(X^\star)}}$\end{small} ~and~
\begin{small}$\min_{R \in \mathcal{O}} \|U_- - U^\star R\|_F  \leq \tfrac{\sigma_r(X^\star)^{1/2}}{10^3 \sqrt{\kappa\tau(X^\star)}}$\end{small}, 
where $X^\star = U^\star U^{\star \top}$, and $\kappa := \tfrac{1 + \delta_{2r}}{1 - \delta_{2r}} > 1$, for $\delta_{2r} \leq \tfrac{1}{10}$, and $\tau(X^\star) := \tfrac{\sigma_1(X^\star)}{\sigma_r(X^\star)} > 1$.
By Lemma \ref{lem:supp_02}, the above imply also that: \begin{small}$\|Z - U^\star R_Z\|_F \leq \left(\tfrac{3}{2} + 2|\mu|\right) \cdot \tfrac{\sigma_r(X^\star)^{1/2}}{10^3\sqrt{\kappa\tau(X^\star)}}$\end{small}.
Then, under RIP assumptions of the mapping $\mathcal{A}$, we have:
\begin{small}
 \begin{align*}
\Big \langle &\mathcal{A}^\dagger(\mathcal{A}(ZZ^\top) - y), (Z - U^\star R_Z)(Z - U^\star R_Z)^\top \Big \rangle \nonumber \\ 
		  &\geq - \Bigg( \theta \sigma_r(X^\star) \cdot \|Z - U^\star R_Z\|_F^2 + \tfrac{10.1}{100} \beta^2 \cdot \widehat{\eta} \cdot \tfrac{(1 + 2|\mu|)^2}{\left(1  - \left(1 + 2|\mu|\right)\tfrac{1}{200}\right)^2} \cdot \|\mathcal{A}^\dagger(\mathcal{A}(ZZ^\top) - y) \cdot Z\|_F^2\Bigg)
\end{align*}
\end{small}
where $$\theta = \tfrac{(1 - \delta_{2r}) \left(1 + (1 + 2|\mu|)\tfrac{1}{200}\right)^2}{10^3} +  (1 + \delta_{2r})  \left(2+\left(1 + 2|\mu|\right) \cdot \tfrac{1}{200}\right) \left(1 + 2|\mu|\right) \cdot \tfrac{1}{200},$$
and $\widehat{\eta}= \tfrac{1}{4((1+\delta_r)\|ZZ^\top\|_2+\|\mathcal{A}^\dagger
\left(\mathcal{A}(ZZ^\top) - y \right)Q_ZQ_Z^\top\|_2)}$.
\end{lemma}

\begin{proof}
First, denote $\Delta := Z - U^\star R_Z$.
Then:
\begin{small}
 \begin{align}
\Big \langle &\mathcal{A}^\dagger(\mathcal{A}(ZZ^\top) - y), (Z - U^\star R_Z)(Z - U^\star R_Z)^\top \Big \rangle \nonumber \\ 
&\stackrel{(i)}{=} \left \langle \mathcal{A}^\dagger(\mathcal{A}(ZZ^\top) - y) \cdot Q_{\Delta} Q_{\Delta}^\top, \Delta_Z \Delta_Z^\top \right \rangle \nonumber \\
&\geq -  \left|\text{Tr}\left(\mathcal{A}^\dagger(\mathcal{A}(ZZ^\top) - y) \cdot Q_{\Delta} Q_{\Delta}^\top \cdot \Delta_Z \Delta_Z^\top\right)\right| \nonumber \\ 
&\stackrel{(ii)}{\geq} - \| \mathcal{A}^\dagger(\mathcal{A}(ZZ^\top) - y) \cdot Q_{\Delta} Q_{\Delta}^\top \|_2 \cdot \text{Tr}( \Delta_Z \Delta_Z^\top)  \nonumber \\
&\stackrel{(iii)}{\geq} - \left( \| \mathcal{A}^\dagger(\mathcal{A}(ZZ^\top) - y) \cdot Q_{Z} Q_{Z}^\top \|_2+  \|\mathcal{A}^\dagger(\mathcal{A}(ZZ^\top) - y) \cdot Q_{U^\star} Q_{U^\star}^\top \|_2 \right) \|Z - U^\star R_Z\|_F^2 \label{proofsr1:eq_11}
\end{align}
\end{small}
Note that  $(i)$ follows from the fact $\Delta_Z =\Delta_Z Q_{\Delta} Q_{\Delta}^\top $, for a matrix $Q$ that spans the row space of $\Delta_Z$, and $(ii)$ follows from $|\text{Tr}(AB)| \leq \|A\|_2 \trace(B)$, for PSD matrix $B$ (Von Neumann's trace inequality~\cite{mirsky1975trace}). 
For the transformation in $(iii)$, we use that fact that the row space of $\Delta_Z$, $\text{\textsc{Span}}(\Delta_Z)$, is a subset of $\text{\textsc{Span}}(Z \cup U^\star)$, as $\Delta_Z$ is a linear combination of $U$ and $U^\star$. 

To bound the first term in equation~\eqref{proofsr1:eq_11}, we observe:
\begin{small}
\begin{align}
&\|\mathcal{A}^\dagger(\mathcal{A}(ZZ^\top) - y) \cdot Q_{Z} Q_{Z}^\top\|_2 \cdot \|Z - U^\star R_Z\|_F^2 \nonumber \\ 
&\quad \quad \stackrel{(i)}{=} \widehat{\eta} \cdot 4 \Big((1 + \delta_{2r})\| ZZ^\top \|_2 \nonumber \\ 
&\quad \quad \quad \quad + \|\mathcal{A}^\dagger(\mathcal{A}(ZZ^\top) - y) \cdot Q_{Z} Q_{Z}^\top\|_2 \Big) \cdot \|\mathcal{A}^\dagger(\mathcal{A}(ZZ^\top) - y)\cdot Q_{Z} Q_{Z}^\top\|_2 \cdot \|Z - U^\star R_Z\|_F^2  \nonumber \\
 &\quad \quad =  \underbrace{4 \widehat{\eta} (1 + \delta_{2r}) \|ZZ^\top\|_2 \| \mathcal{A}^\dagger(\mathcal{A}(ZZ^\top) - y) \cdot Q_{Z} Q_{Z}^\top\|_2 \cdot \|Z - U^\star R_Z\|_F^2}_{:=A} \nonumber \\ &\quad \quad \quad \quad \quad \quad \quad \quad + 4 \widehat{\eta} \|\mathcal{A}^\dagger(\mathcal{A}(ZZ^\top) - y) \cdot Q_{Z} Q_{Z}^\top\|_2^2 \cdot  \|Z - U^\star R_Z\|_F^2 \nonumber
\end{align} 
\end{small}
where $(i)$ is due to the definition of $\widehat{\eta}$.

To bound term $A$, we observe that $\|\mathcal{A}^\dagger(\mathcal{A}(ZZ^\top) - y) \cdot Q_{Z} Q_{Z}^\top\|_2 \leq \tfrac{(1 - \delta_{2r}) \sigma_r(ZZ^\top)}{10^3}$ or $\|\mathcal{A}^\dagger(\mathcal{A}(ZZ^\top) - y) \cdot Q_{Z} Q_{Z}^\top\|_2 \geq \tfrac{(1 - \delta_{2r}) \sigma_r(ZZ^\top)}{10^3}$.
This results into bounding $A$ as follows:
\begin{small}
\begin{align*}
&4 \widehat{\eta} (1 + \delta_{2r}) \| ZZ^\top \|_2 \|\mathcal{A}^\dagger(\mathcal{A}(ZZ^\top) - y) \cdot Q_{Z} Q_{Z}^\top\|_2 \cdot \|Z - U^\star R_Z\|_F^2 \\ 
&\quad \quad \leq \max \Big\{\tfrac{4\cdot \widehat{\eta} \cdot (1 + \delta_{2r}) \| ZZ^\top \|_2 \cdot (1 - \delta_{2r}) \sigma_r(ZZ^\top)}{10^3} \cdot \|Z - U^\star R_Z\|_F^2 , \\ 
&\quad \quad \quad \quad \quad \quad \quad \quad \widehat{\eta} \cdot 4 \cdot 10^3 \kappa \tau(ZZ^\top) \|\mathcal{A}^\dagger(\mathcal{A}(ZZ^\top) - y) \cdot Q_{Z} Q_{Z}^\top\|_2^2 \cdot  \|Z - U^\star R_Z\|_F^2 \Big\} \\
&\quad \quad \leq \tfrac{4\cdot \widehat{\eta} \cdot (1 - \delta_{2r}^2) \| ZZ^\top \|_2 \cdot \sigma_r(ZZ^\top)}{10^3} \cdot \|Z - U^\star R_Z\|_F^2 \\ 
&\quad \quad \quad \quad \quad \quad \quad \quad + \widehat{\eta} \cdot 4 \cdot 10^3 \kappa \tau(ZZ^\top) \|\mathcal{A}^\dagger(\mathcal{A}(ZZ^\top) - y) \cdot Q_{Z} Q_{Z}^\top\|_2^2 \cdot  \|Z - U^\star R_Z\|_F^2.
\end{align*}
\end{small}

Combining the above inequalities, we obtain:
\begin{small}
\begin{align}
\|\mathcal{A}^\dagger(\mathcal{A}(ZZ^\top) - y) &\cdot Q_{Z} Q_{Z}^\top\|_2 \cdot \|Z - U^\star R_Z\|_F^2 \nonumber \\ 
	&\stackrel{(i)}{\leq}  \tfrac{(1 - \delta_{2r}) \sigma_{r}(ZZ^\top)}{10^3} \cdot \|Z - U^\star R_Z\|_F^2 \nonumber \\ 
	&\quad \quad  \quad + (10^3 \kappa \tau(ZZ^\top)+1) \cdot  4 \cdot \widehat{\eta} \|\mathcal{A}^\dagger(\mathcal{A}(ZZ^\top) - y) \cdot Q_{Z} Q_{Z}^\top\|_2^2 \cdot \|Z - U^\star R_Z\|_F^2  \nonumber \\
	&\stackrel{(ii)}{\leq}  \tfrac{(1 - \delta_{2r}) \sigma_{r}(ZZ^\top)}{10^3} \cdot \|Z - U^\star R_Z\|_F^2 \nonumber \\
	&\quad \quad  \quad + (10^3 \beta^2 \kappa \tau(X^\star)+1) \cdot  4 \cdot \widehat{\eta} \|\mathcal{A}^\dagger(\mathcal{A}(ZZ^\top) - y) \cdot Q_{Z} Q_{Z}^\top\|_2^2 \cdot \tfrac{\left(\tfrac{3}{2}+2|\mu|\right)^2}{\kappa \tau(X^\star)} \tfrac{1}{10^6} \sigma_{r}(X^\star)  \nonumber \\
	&\stackrel{(iii)}{\leq}  \tfrac{(1 - \delta_{2r}) \sigma_{r}(ZZ^\top)}{10^3} \cdot \|Z - U^\star R_Z\|_F^2 \nonumber \\ 
	&\quad\quad \quad + 4 \cdot 1001 \beta^2 \cdot \widehat{\eta} \cdot \|\mathcal{A}^\dagger(\mathcal{A}(ZZ^\top) - y) \cdot Q_{Z} Q_{Z}^\top\|_2^2 \cdot \tfrac{\left(\tfrac{3}{2} + 2|\mu|\right)^2}{10^6\left(1  - \left(\tfrac{3}{2} + 2|\mu|\right)\tfrac{1}{10^3}\right)^2}\sigma_{r}(ZZ^\top)   \nonumber \\ 
	&\stackrel{(iv)}{\leq}  \tfrac{(1 - \delta_{2r}) \sigma_{r}(ZZ^\top)}{10^3} \cdot \|Z - U^\star R_Z\|_F^2 \nonumber \\ 
	&\quad\quad \quad + 4 \cdot 1001 \beta^2 \cdot \widehat{\eta} \cdot \tfrac{\left(\tfrac{3}{2} + 2|\mu|\right)^2}{10^6 \left(1  - \left(\tfrac{3}{2} + 2|\mu|\right)\tfrac{1}{10^3}\right)^2} \cdot \|\mathcal{A}^\dagger(\mathcal{A}(ZZ^\top) - y) \cdot Z\|_F^2   \nonumber \\ 
	&\stackrel{(v)}{\leq}  \tfrac{(1 - \delta_{2r}) \left(1 + (\tfrac{3}{2} + 2|\mu|)\tfrac{1}{10^3}\right)^2 \sigma_{r}(X^\star)}{10^3} \cdot \|Z - U^\star R_Z\|_F^2 \nonumber \\ 
	&\quad\quad \quad + \tfrac{1}{200} \beta^2 \cdot \widehat{\eta} \cdot \tfrac{\left(\tfrac{3}{2} + 2|\mu|\right)^2}{\left(1  - \left(\tfrac{3}{2} + 2|\mu|\right)\tfrac{1}{10^3}\right)^2} \cdot \|\mathcal{A}^\dagger(\mathcal{A}(ZZ^\top) - y) \cdot Z\|_F^2   \nonumber 
\end{align}
\end{small}
where $(i)$ follows from $\widehat{\eta} \leq \tfrac{1}{4 (1+\delta_{2r}) \| ZZ^\top \|_2}$, 
$(ii)$ is due to Corollary \ref{cor:1}, bounding $\|Z -  U^\star R_Z\|_F \leq \rho \sigma_{r}(X^\star)^{1/2}$, where $\rho := \left(\tfrac{3}{2} + 2|\mu|\right) \tfrac{1}{10^3 \sqrt{\kappa \tau(X^\star)}}$ by Lemma \ref{lem:supp_02},
$(iii)$ is due to $(10^3 \beta^2 \kappa \tau(X^\star)+1) \leq 1001 \beta^2 \kappa \tau(X^\star)$, and by Corollary \ref{cor:2}, 
$(iv)$ is due to the fact $\sigma_{r}(ZZ^\top)\|\mathcal{A}^\dagger(\mathcal{A}(ZZ^\top) - y) \cdot Q_{Z} Q_{Z}^\top\|_2^2 \leq  \| \mathcal{A}^\dagger(\mathcal{A}(ZZ^\top) - y) Z\|_F^2$,
and $(v)$ is due to Corollary \ref{cor:2}.

Next, we bound the second term in equation~\eqref{proofsr1:eq_11}:
\begin{align}
\|\mathcal{A}^\dagger(\mathcal{A}(ZZ^\top) - y) &\cdot Q_{U^\star} Q_{U^\star}^\top\|_2 \cdot \|Z - U^\star R_Z\|_F^2 \nonumber \\ 
&\stackrel{(i)}{\leq}  \|\mathcal{A}^\dagger(\mathcal{A}(ZZ^\top) - y) - \mathcal{A}^\dagger(\mathcal{A}(X^\star) - y) \|_2 \cdot \|Z - U^\star R_Z\|_F^2 \nonumber \\
&\stackrel{(ii)}{\leq} (1 + \delta_{2r}) \cdot \|ZZ^\top - U^\star U^{\star \top}\|_F \cdot \|Z - U^\star R_Z\|_F^2 \nonumber \\
&\stackrel{(iii)}{\leq} (1 + \delta_{2r})  (2+\rho) \cdot \rho \cdot \sigma_1(U^\star) \cdot \sigma_r(U^\star) \cdot \|Z - U^\star R_Z\|_F^2 \nonumber \\
&\stackrel{(iv)}{\leq}  (1 + \delta_{2r})  (2+\rho) \left(\tfrac{3}{2} + 2|\mu|\right) \cdot \tfrac{1}{10^3} \sigma_r(X^\star) \cdot \|Z - U^\star R_Z\|_F^2 \nonumber \\
&\leq (1 + \delta_{2r})  \left(2+\left(\tfrac{3}{2} + 2|\mu|\right) \cdot \tfrac{1}{10^3}\right) \left(\tfrac{3}{2} + 2|\mu|\right) \cdot \tfrac{1}{10^3} \sigma_r(X^\star) \cdot \|Z - U^\star R_Z\|_F^2, \nonumber
\end{align}
where $(i)$ follows from $\|\mathcal{A}^\dagger(\mathcal{A}(ZZ^\top) - y) \cdot Q_{U^\star} Q_{U^\star}^\top\|_2 \leq \|\mathcal{A}^\dagger(\mathcal{A}(ZZ^\top) - y)\|_2$ and $\mathcal{A}^\dagger(\mathcal{A}(X^\star) - y) =0$, 
$(ii)$ is due to smoothness of $f$ and the RIP constants, $(iii)$ follows from \cite[Lemma 18]{bhojanapalli2016dropping}, for $\rho = \left(\tfrac{3}{2} + 2|\mu|\right) \cdot \tfrac{1}{10^3 \sqrt{\kappa \tau(X^\star)}}$, 
$(iv)$ follows from substituting $\rho$ above, and observing that $\tau(X^\star) = \sigma_1(U^\star)^2 / \sigma_r(U^\star)^2 > 1$ and $\kappa = (1+\delta_{2r})/(1-\delta_{2r}) > 1$.

Combining the above we get:
\begin{small}
 \begin{align*}
\Big \langle &\mathcal{A}^\dagger(\mathcal{A}(ZZ^\top) - y), (Z - U^\star R_Z)(Z - U^\star R_Z)^\top \Big \rangle \nonumber \\ 
		  &\geq - \Bigg( \theta \sigma_r(X^\star) \cdot \|Z - U^\star R_Z\|_F^2 + \tfrac{1}{200} \beta^2 \cdot \widehat{\eta} \cdot \tfrac{\left(\tfrac{3}{2} + 2|\mu| \right)^2}{\left(1  - \left(\tfrac{3}{2} + 2|\mu|\right)\tfrac{1}{10^3}\right)^2} \cdot \|\mathcal{A}^\dagger(\mathcal{A}(ZZ^\top) - y) \cdot Z\|_F^2\Bigg)
\end{align*}
\end{small}
where $\theta = \tfrac{(1 - \delta_{2r}) \left(1 + \left(\tfrac{3}{2} + 2|\mu| \right)\tfrac{1}{10^3}\right)^2}{10^3} +  (1 + \delta_{2r})  \left(2+\left(\tfrac{3}{2} + 2|\mu|\right) \cdot \tfrac{1}{10^3}\right) \left(\tfrac{3}{2} + 2|\mu|\right) \cdot \tfrac{1}{10^3}$.
\end{proof}

\begin{lemma}{\label{lem:001}}
Under identical assumptions with Lemma \ref{lem:000}, the following inequality holds: 
\begin{small}
\begin{align*}
\Big \langle \mathcal{A}^\dagger( \mathcal{A}(ZZ^\top) - y), ZZ^\top - U^\star U^{\star \top}\Big \rangle \geq 1.1172 \eta \left \|\mathcal{A}^\dagger( \mathcal{A}(ZZ^\top) - y)  Z \right \|_F^2 + \tfrac{1 - \delta_{2r}}{2} \|U^\star U^{\star \top} - ZZ^\top\|_F^2
\end{align*}
\end{small}
\end{lemma}

\begin{proof}
By smoothness assumption of the objective, based on the RIP assumption, we have:
\begin{small}
\begin{align*}
\tfrac{1}{2} \|\mathcal{A}(ZZ^\top) - y\|_2^2 &\geq \tfrac{1}{2} \|\mathcal{A}(U_{+} U_{+}^\top) - y\|_2^2 \\ 
								   &\quad \quad - \left \langle \mathcal{A}^\dagger( \mathcal{A}(ZZ^\top) - y), U_{+} U_{+}^\top - ZZ^\top \right \rangle - \tfrac{1+\delta_{2r}}{2} \|U_{+} U_{+}^\top - ZZ^\top\|_F^2 \Rightarrow \\
\tfrac{1}{2} \|\mathcal{A}(ZZ^\top) - y\|_2^2 &\geq \tfrac{1}{2} \|\mathcal{A}(U^\star U^{\star \top}) - y\|_2^2 \\ 
								   &\quad \quad - \left \langle \mathcal{A}^\dagger( \mathcal{A}(ZZ^\top) - y), U_{+} U_{+}^\top - ZZ^\top \right \rangle - \tfrac{1+\delta_{2r}}{2} \|U_{+} U_{+}^\top - ZZ^\top\|_F^2								   
\end{align*}
\end{small}
due to the optimality $\|\mathcal{A}(U^\star U^{\star \top}) - y\|_2^2 = 0 \leq \|\mathcal{A}(VV^\top) - y\|_2^2$, for any $V \in \mathbb{R}^{n \times r}$.
Also, by the restricted strong convexity with RIP, we get:
\begin{small}
\begin{align*}
\tfrac{1}{2} \|\mathcal{A}(U^\star U^{\star \top}) - y\|_2^2 &\geq \tfrac{1}{2} \|\mathcal{A}(ZZ^\top) - y\|_2^2 \\ 
								   &\quad \quad + \left \langle \mathcal{A}^\dagger( \mathcal{A}(ZZ^\top) - y), U^\star U^{\star \top} - ZZ^\top \right \rangle + \tfrac{1-\delta_{2r}}{2} \|U^\star U^{\star \top} - ZZ^\top\|_F^2					   
\end{align*}
\end{small}
Adding the two inequalities, we obtain:
\begin{small}
\begin{align*}
\left \langle \mathcal{A}^\dagger( \mathcal{A}(ZZ^\top) - y), ZZ^\top - U^\star U^{\star \top}\right \rangle &\geq \left \langle \mathcal{A}^\dagger( \mathcal{A}(ZZ^\top) - y), ZZ^\top - U_{+} U_{+}^\top\right \rangle \\
&\quad \quad - \tfrac{1+\delta_{2r}}{2} \|U_+ U_+^\top - ZZ^\top \|_F^2 + \tfrac{1 - \delta_{2r}}{2} \|U^\star U^{\star \top} - ZZ^\top\|_F^2
\end{align*}
\end{small}
To proceed we observe:
\begin{small}
\begin{align*}
U_+ U_+^\top &= \left(Z - \eta \mathcal{A}^\dagger \left(\mathcal{A}(ZZ^\top) - y \right)Z \right) \cdot \left(Z - \eta \mathcal{A}^\dagger \left(\mathcal{A}(ZZ^\top) - y \right)Z \right)^\top \\
		      &= ZZ^\top - \eta ZZ^\top \cdot \mathcal{A}^\dagger \left(\mathcal{A}(ZZ^\top) - y \right) - \eta \mathcal{A}^\dagger \left(\mathcal{A}(ZZ^\top) - y \right) \cdot ZZ^\top \\
		      & \quad \quad + \eta^2 \mathcal{A}^\dagger \left(\mathcal{A}(ZZ^\top) - y \right) \cdot ZZ^\top \cdot \mathcal{A}^\dagger \left(\mathcal{A}(ZZ^\top) - y \right) \\
		      &\stackrel{(i)}{=} ZZ^\top - \left(I - \tfrac{\eta}{2} Q_Z Q_Z^\top  \mathcal{A}^\dagger \left(\mathcal{A}(ZZ^\top) - y \right) \right) \cdot \eta ZZ^\top \cdot \mathcal{A}^\dagger \left(\mathcal{A}(ZZ^\top) - y \right) \\ &\quad \quad - \eta \mathcal{A}^\dagger \left(\mathcal{A}(ZZ^\top) - y \right) \cdot ZZ^\top  \cdot \left(I - \tfrac{\eta}{2} Q_Z Q_Z^\top \mathcal{A}^\dagger \left(\mathcal{A}(ZZ^\top) - y \right) \right) 
\end{align*}
\end{small}
where $(i)$ is due to the fact $\mathcal{A}^\dagger \left(\mathcal{A}(ZZ^\top) - y \right) \cdot ZZ^\top \cdot \mathcal{A}^\dagger \left(\mathcal{A}(ZZ^\top) - y \right) = \mathcal{A}^\dagger \left(\mathcal{A}(ZZ^\top) - y \right) \cdot Q_ZQ_Z^\top \cdot ZZ^\top \cdot Q_ZQ_Z^\top \cdot \mathcal{A}^\dagger \left(\mathcal{A}(ZZ^\top) - y \right)$, for $Q_Z$ a basis matrix whose columns span the column space of $Z$; also, $I$ is the identity matrix whose dimension is apparent from the context.
Thus:
\begin{small}
\begin{align*}
\tfrac{\eta}{2} Q_Z Q_Z^\top \mathcal{A}^\dagger \left(\mathcal{A}(ZZ^\top) - y \right) \preceq \tfrac{\widehat{\eta}}{2} Q_Z Q_Z^\top \mathcal{A}^\dagger \left(\mathcal{A}(ZZ^\top) - y \right),
\end{align*}
\end{small}
and, hence,
\begin{small}
\begin{align*}
I - \tfrac{\eta}{2} Q_Z Q_Z^\top \mathcal{A}^\dagger \left(\mathcal{A}(ZZ^\top) - y \right) \succeq I - \tfrac{\widehat{\eta}}{2} Q_Z Q_Z^\top \mathcal{A}^\dagger \left(\mathcal{A}(ZZ^\top) - y \right).
\end{align*}
\end{small}
Define $\Psi = I - \tfrac{\eta}{2} Q_Z Q_Z^\top \mathcal{A}^\dagger \left(\mathcal{A}(ZZ^\top) - y \right).$
Then, using the definition of $\widehat{\eta}$, we know that $\widehat{\eta} \leq \tfrac{1}{4\|Q_Z Q_Z^\top \mathcal{A}^\dagger \left(\mathcal{A}(ZZ^\top) - y \right)\|_2}$, and thus:
\begin{small}
\begin{align*}
\Psi \succ 0, \quad \sigma_1(\Psi) \leq 1 + \tfrac{1}{4}, \quad \text{and} \quad \sigma_n(\Psi) \geq 1 - \tfrac{1}{4}.
\end{align*}
\end{small}
Going back to the main recursion and using the above expression for $U_+ U_+^\top$, we have:
\begin{small}
\begin{align*}
\Big \langle \mathcal{A}^\dagger( \mathcal{A}(ZZ^\top) - y), &ZZ^\top - U^\star U^{\star \top}\Big \rangle - \tfrac{1 - \delta_{2r}}{2} \|U^\star U^{\star \top} - ZZ^\top\|_F^2 \\
			&\geq \left \langle \mathcal{A}^\dagger( \mathcal{A}(ZZ^\top) - y), ZZ^\top - U_{+} U_{+}^\top \right \rangle - \tfrac{1+\delta_{2r}}{2} \|U_{+} U_{+}^\top - ZZ^\top\|_F^2 \\ 
			&\stackrel{(i)}{\geq} 2 \eta \left \langle \mathcal{A}^\dagger( \mathcal{A}(ZZ^\top) - y),  \mathcal{A}^\dagger( \mathcal{A}(ZZ^\top) - y) \cdot ZZ^\top \cdot \Psi \right \rangle \\
			 &\quad \quad - \tfrac{1+\delta_{2r}}{2} \|2\eta  \mathcal{A}^\dagger( \mathcal{A}(ZZ^\top) - y) \cdot ZZ^\top \cdot \Psi\|_F^2 \\			
			 &\stackrel{(ii)}{\geq} \tfrac{7}{4} \eta \left \|\mathcal{A}^\dagger( \mathcal{A}(ZZ^\top) - y)  Z \right \|_F^2 \\ 
			 &\quad \quad - 2(1 + \delta_{2r}) \eta^2 \left \|\mathcal{A}^\dagger( \mathcal{A}(ZZ^\top) - y)  Z \right \|_F^2 \cdot \|Z\|_2^2 \cdot \|\Psi\|_2^2 \\
			 &\stackrel{(iii)}{\geq} \tfrac{7}{4} \eta \left \|\mathcal{A}^\dagger( \mathcal{A}(ZZ^\top) - y)  Z \right \|_F^2 \\ 
			 &\quad \quad - 2(1 + \delta_{2r}) \eta^2 \left \|\mathcal{A}^\dagger( \mathcal{A}(ZZ^\top) - y)  Z \right \|_F^2 \cdot \|Z\|_2^2 \cdot \left(\tfrac{9}{8}\right)^2 \\
			 &= \tfrac{7}{4} \eta \left \|\mathcal{A}^\dagger( \mathcal{A}(ZZ^\top) - y)  Z \right \|_F^2 \cdot \left(1 - 2(1 + \delta_{2r}) \eta  \cdot \|Z\|_2^2 \cdot \left(\tfrac{9}{8}\right)^2 \cdot \tfrac{4}{7}\right) \\
			 &\stackrel{(iv)}{\geq} \tfrac{7}{4} \eta \left \|\mathcal{A}^\dagger( \mathcal{A}(ZZ^\top) - y)  Z \right \|_F^2 \cdot \left(1 - 2(1 + \delta_{2r}) \widehat{\eta}  \cdot \|Z\|_2^2 \cdot \left(\tfrac{9}{8}\right)^2 \cdot \tfrac{4}{7}\right) \\
			 &\stackrel{(v)}{\geq} \tfrac{7}{4} \eta \left \|\mathcal{A}^\dagger( \mathcal{A}(ZZ^\top) - y)  Z \right \|_F^2 \cdot \left(1 - \tfrac{2 \cdot \left(\tfrac{9}{8}\right)^2}{7}\right) \\
			 &= 1.1172 \eta \left \|\mathcal{A}^\dagger( \mathcal{A}(ZZ^\top) - y)  Z \right \|_F^2
\end{align*}
\end{small}
where $(i)$ is due to the symmetry of the objective;
$(ii)$ is due to Cauchy-Schwartz inequality and the fact:
\begin{small}
\begin{align*}
\Big \langle  \mathcal{A}^\dagger( \mathcal{A}(ZZ^\top) - y), & ~\mathcal{A}^\dagger( \mathcal{A}(ZZ^\top) - y) \cdot ZZ^\top \cdot \Psi \Big \rangle \\ 
											   &= \Big \langle  \mathcal{A}^\dagger( \mathcal{A}(ZZ^\top) - y),  \mathcal{A}^\dagger( \mathcal{A}(ZZ^\top) - y) \cdot ZZ^\top \Big \rangle  \\ &\quad \quad - \tfrac{\eta}{2} \Big \langle  \mathcal{A}^\dagger( \mathcal{A}(ZZ^\top) - y),  \mathcal{A}^\dagger( \mathcal{A}(ZZ^\top) - y) \cdot ZZ^\top \cdot \mathcal{A}^\dagger( \mathcal{A}(ZZ^\top) - y) \Big \rangle \\
											   &\stackrel{(i)}{\geq} \Big \langle  \mathcal{A}^\dagger( \mathcal{A}(ZZ^\top) - y),  \mathcal{A}^\dagger( \mathcal{A}(ZZ^\top) - y) \cdot ZZ^\top \Big \rangle  \\ &\quad \quad - \tfrac{\widehat{\eta}}{2} \Big \langle  \mathcal{A}^\dagger( \mathcal{A}(ZZ^\top) - y),  \mathcal{A}^\dagger( \mathcal{A}(ZZ^\top) - y) \cdot ZZ^\top \cdot \mathcal{A}^\dagger( \mathcal{A}(ZZ^\top) - y) \Big \rangle \\
											   &\geq \left( 1 - \tfrac{\widehat{\eta}}{2} \|Q_Z Q_Z^\top \mathcal{A}^\dagger( \mathcal{A}(ZZ^\top) - y) \|_2^2 \right) \cdot \left \|\mathcal{A}^\dagger( \mathcal{A}(ZZ^\top) - y)  Z \right \|_F^2 \\
											   &\geq \left( 1 - \tfrac{1}{4}\right) \left \|\mathcal{A}^\dagger( \mathcal{A}(ZZ^\top) - y)  Z \right \|_F^2
\end{align*}
\end{small}
where $(i)$ is due to $\eta \leq \widehat{\eta}$, and the last inequality comes from the definition of the $\widehat{\eta}$ and its upper bound;
$(iii)$ is due to the upper bound on $\|\Psi\|_2$ above;
$(iv)$ is due to $\eta \leq \widehat{\eta}$;
$(v)$ is due to $\widehat{\eta} \leq \tfrac{1}{4 (1 + \delta_{2r}) \|ZZ^\top\|_2}$.

The above lead to the desiderata:
\begin{small}
\begin{align*}
\Big \langle \mathcal{A}^\dagger( \mathcal{A}(ZZ^\top) - y), ZZ^\top - U^\star U^{\star \top}\Big \rangle \geq 1.1172 \eta \left \|\mathcal{A}^\dagger( \mathcal{A}(ZZ^\top) - y)  Z \right \|_F^2 + \tfrac{1 - \delta_{2r}}{2} \|U^\star U^{\star \top} - ZZ^\top\|_F^2
\end{align*}
\end{small}
\end{proof}

\section{Detailed proof of Theorem \ref{thm:00}} 
We first denote $U_{+} \equiv U_{i+1}$, $U \equiv U_i$, $U_- \equiv U_{i - 1}$ and $Z \equiv Z_i$.
Let us start with the following equality. 
For $R_Z \in \mathcal{O}$ as the minimizer of $\min_{R \in \mathcal{O}} \|Z - U^\star R\|_F$,
we have:
\begin{align*}
\|U_{+} - U^\star R_Z\|_F^2 &= \|U_{+} - Z + Z - U^\star R_Z\|_F^2 \\
					        &= \|U_{+} - Z\|_F^2 + \|Z - U^\star R_Z\|_F^2  - 2 \langle U_{+} - Z, U^\star R_Z - Z \rangle
\end{align*}
The proof focuses on how to bound the last part on the right-hand side.
By definition of $U_{+}$, we get:
\begin{align*}
\langle U_{+} - Z, U^\star R_Z - Z \rangle &= \left \langle Z - \eta \mathcal{A}^\dagger\left(\mathcal{A}(ZZ^\top) - y \right) Z - Z, U^\star R_Z - Z \right \rangle \\
						          &= \eta \left \langle \mathcal{A}^\dagger\left(\mathcal{A}(ZZ^\top) - y \right) Z, Z - U^\star R_Z \right \rangle
\end{align*}
Observe the following:
\begin{small}
\begin{align*}
\left \langle \mathcal{A}^\dagger\left(\mathcal{A}(ZZ^\top) - y \right) Z, Z - U^\star R_Z \right \rangle &= \left \langle \mathcal{A}^\dagger\left(\mathcal{A}(ZZ^\top) - y \right), ZZ^\top  - U^\star R_Z Z^\top \right \rangle \\
																		&= \left \langle \mathcal{A}^\dagger\left(\mathcal{A}(ZZ^\top) - y \right), ZZ^\top - \tfrac{1}{2} U^\star U^{\star \top} + \tfrac{1}{2} U^\star U^{\star \top} - U^\star R_Z Z^\top \right \rangle \\
																		&= \tfrac{1}{2} \left \langle \mathcal{A}^\dagger\left(\mathcal{A}(ZZ^\top) - y \right), ZZ^\top - U^\star U^{\star \top} \right \rangle \\
																		&\quad \quad + \left \langle \mathcal{A}^\dagger\left(\mathcal{A}(ZZ^\top) - y \right), \tfrac{1}{2} (ZZ^\top + U^\star U^{\star \top}) - U^\star R_Z Z^\top \right \rangle \\
																																				&= \tfrac{1}{2} \left \langle \mathcal{A}^\dagger\left(\mathcal{A}(ZZ^\top) - y \right), ZZ^\top - U^\star U^{\star \top} \right \rangle \\
																		&\quad \quad + \tfrac{1}{2} \left \langle \mathcal{A}^\dagger\left(\mathcal{A}(ZZ^\top) - y \right), (Z - U^\star R_Z)(Z - U^\star R_Z)^\top \right \rangle
\end{align*}
\end{small}
By Lemmata \ref{lem:000} and \ref{lem:001}, we have:
\begin{small}
\begin{align*}
\|U_+ - U^\star R_Z\|_F^2 &= \|U_{+} - Z\|_F^2 + \|Z - U^\star R_Z\|_F^2  - 2 \langle U_{+} - Z, U^\star R_Z - Z \rangle \\
					       &= \eta^2 \|\mathcal{A}^\dagger\left(\mathcal{A}(ZZ^\top) - y \right) Z \|_F^2 + \|Z - U^\star R_Z\|_F^2  \\ 
					       &\quad \quad -\eta \left \langle \mathcal{A}^\dagger\left(\mathcal{A}(ZZ^\top) - y \right), ZZ^\top - U^\star U^{\star \top} \right \rangle \\
					       &\quad \quad \quad \quad - \eta \left \langle \mathcal{A}^\dagger\left(\mathcal{A}(ZZ^\top) - y \right), (Z - U^\star R_Z)(Z - U^\star R_Z)^\top \right \rangle \\
					       &\leq \eta^2 \|\mathcal{A}^\dagger\left(\mathcal{A}(ZZ^\top) - y \right) Z \|_F^2 + \|Z - U^\star R_Z\|_F^2  \\ 
					       &\quad \quad - 1.1172 \eta^2 \left \|\mathcal{A}^\dagger( \mathcal{A}(ZZ^\top) - y)  Z \right \|_F^2 - \eta \tfrac{1 - \delta_{2r}}{2} \|U^\star U^{\star \top} - ZZ^\top\|_F^2 \\
					       &\quad \quad \quad \quad + \eta \Bigg( \theta \sigma_r(X^\star) \cdot \|Z - U^\star R_Z\|_F^2 \\ 
					       &\quad \quad \quad \quad \quad \quad + \tfrac{1}{200} \beta^2 \cdot \widehat{\eta} \cdot \tfrac{\left(\tfrac{3}{2} + 2|\mu|\right)^2}{\left(1  - \left(\tfrac{3}{2} + 2|\mu|\right)\tfrac{1}{10^3}\right)^2} \cdot \|\mathcal{A}^\dagger(\mathcal{A}(ZZ^\top) - y) \cdot Z\|_F^2\Bigg)
\end{align*}
\end{small}

Next, we use the following lemma:
\begin{lemma}\cite[Lemma 5.4]{tu2016low}{\label{lem:tu}}
For any $W, V \in \mathbb{R}^{n \times r}$, the following holds:
\begin{small}
\begin{align*}
\|WW^\top - VV^\top\|_F^2 \geq 2 (\sqrt{2} - 1) \cdot \sigma_r(VV^\top) \cdot \min_{R \in \mathcal{O}} \|W - VR\|_F^2.
\end{align*}
\end{small}
\end{lemma}

From Lemma \ref{lem:tu}, the quantity  $\|U^\star U^{\star \top} - ZZ^\top\|_F^2$ satisfies:
\begin{small}
\begin{align*}
\|U^\star U^{\star \top} - ZZ^\top\|_F^2 \geq 2 (\sqrt{2} - 1) \cdot \sigma_r(X^\star) \cdot \min_{R \in \mathcal{O}} \|Z - U^\star R\|_F^2 = 2 (\sqrt{2} - 1) \cdot \sigma_r(X^\star) \cdot \|Z - U^\star R_Z\|_F^2 ,
\end{align*}
\end{small}
which, in our main recursion, results in:
\begin{small}
\begin{align*}
\|U_+ - U^\star R_Z\|_F^2 &\leq \eta^2 \|\mathcal{A}^\dagger\left(\mathcal{A}(ZZ^\top) - y \right) Z \|_F^2 + \|Z - U^\star R_Z\|_F^2  \\ 
					       &\quad \quad - 1.1172 \eta^2 \left \|\mathcal{A}^\dagger( \mathcal{A}(ZZ^\top) - y)  Z \right \|_F^2 - \eta (\sqrt{2}-1)(1 - \delta_{2r}) \sigma_r(X^\star) \|Z - U^\star R_Z\|_F^2 \\
					       &\quad \quad \quad \quad + \eta \Bigg( \theta \sigma_r(X^\star) \cdot \|Z - U^\star R_Z\|_F^2 \\ 
					       &\quad \quad \quad \quad \quad \quad + \tfrac{1}{200} \beta^2 \cdot \widehat{\eta} \cdot \tfrac{\left(\tfrac{3}{2} + 2|\mu|\right)^2}{\left(1  - \left(\tfrac{3}{2} + 2|\mu|\right)\tfrac{1}{10^3}\right)^2} \cdot \|\mathcal{A}^\dagger(\mathcal{A}(ZZ^\top) - y) \cdot Z\|_F^2\Bigg) \\
					       &\stackrel{(i)}{\leq} \eta^2 \|\mathcal{A}^\dagger\left(\mathcal{A}(ZZ^\top) - y \right) Z \|_F^2 + \|Z - U^\star R_Z\|_F^2  \\ 
					       &\quad \quad - 1.1172 \eta^2 \left \|\mathcal{A}^\dagger( \mathcal{A}(ZZ^\top) - y)  Z \right \|_F^2 - \eta (\sqrt{2}-1)(1 - \delta_{2r}) \sigma_r(X^\star) \|Z - U^\star R_Z\|_F^2 \\
					       &\quad \quad \quad \quad + \eta \Bigg( \theta \sigma_r(X^\star) \cdot \|Z - U^\star R_Z\|_F^2 \\ 
					       &\quad \quad \quad \quad \quad \quad + \tfrac{1}{200} \beta^2 \cdot \tfrac{10}{9} \eta \cdot \tfrac{\left(\tfrac{3}{2} + 2|\mu|\right)^2}{\left(1  - \left(\tfrac{3}{2} + 2|\mu|\right)\tfrac{1}{10^3}\right)^2} \cdot \|\mathcal{A}^\dagger(\mathcal{A}(ZZ^\top) - y) \cdot Z\|_F^2\Bigg) \\					       
					       &\stackrel{(ii)}{=} \left(1 + \tfrac{1}{200} \beta^2 \cdot \tfrac{10}{9} \cdot \tfrac{\left(\tfrac{3}{2}+ 2|\mu|\right)^2}{\left(1  - \left(\tfrac{3}{2} + 2|\mu|\right)\tfrac{1}{10^3}\right)^2} - 1.1172\right) \eta^2 \|\mathcal{A}^\dagger(\mathcal{A}(ZZ^\top) - y) \cdot Z\|_F^2  \\ 
					       &\quad \quad \quad \quad +\left(1 + \theta \sigma_r(X^\star) - \eta (\sqrt{2}-1)(1 - \delta_{2r}) \sigma_r(X^\star) \right) \|Z - U^\star R_Z\|_F^2 
\end{align*}
\end{small}
where $(i)$ is due to Lemma \ref{lem:equiveta}, 
and $(ii)$ is due to the definition of $U_{+}$.

Under the facts that \begin{small}$\mu = \frac{\sigma_r(X^\star)^{1/2}}{10^3 \sqrt{\kappa\tau(X^\star)}} \cdot \frac{\varepsilon}{2 \cdot \sigma_1(X^\star)^{1/2} \cdot r}$\end{small}, for $\varepsilon \in (0, 1)$ user-defined, and $\delta_{2r} \leq \tfrac{1}{10}$, the main constant quantities in our proof so far simplify into:
\begin{align*}
\beta = \frac{1 + \left(\tfrac{3}{2} + 2|\mu| \right) \cdot \tfrac{1}{10^3}}{1 - \left(\tfrac{3}{2} + 2|\mu| \right) \cdot \tfrac{1}{10^3}} = 1.003, \quad \text{and} \quad \beta^2 = 1.006,
\end{align*}
by Corollary \ref{cor:1}. 
Thus:
\begin{align*}
1 + \tfrac{1}{200} \beta^2 \cdot \tfrac{10}{9} \cdot \tfrac{\left(\tfrac{3}{2}+ 2|\mu|\right)^2}{\left(1  - \left(\tfrac{3}{2} + 2|\mu|\right)\tfrac{1}{10^3}\right)^2} - 1.1172 \leq -0.1046,
\end{align*}
and our recursion becomes:
\begin{small}
\begin{align*}
\|U_+ - U^\star R_Z\|_F^2 &\leq -0.1046 \cdot \eta^2 \cdot \|\mathcal{A}^\dagger(\mathcal{A}(ZZ^\top) - y) \cdot Z\|_F^2  \\ 
					       &\quad \quad \quad \quad +\left(1 + \theta \sigma_r(X^\star) - \eta (\sqrt{2}-1)(1 - \delta_{2r}) \sigma_r(X^\star) \right) \|Z - U^\star R_Z\|_F^2 
\end{align*}
\end{small}
Finally, 
\begin{small}
\begin{align*}
\theta &= \tfrac{(1 - \delta_{2r}) \left(1 + \left(\tfrac{3}{2} + 2|\mu| \right)\tfrac{1}{10^3}\right)^2}{10^3} +  (1 + \delta_{2r})  \left(2+\left(\tfrac{3}{2} + 2|\mu|\right) \cdot \tfrac{1}{10^3}\right) \left(\tfrac{3}{2} + 2|\mu|\right) \cdot \tfrac{1}{10^3} \\ 
	  &\stackrel{(i)}{=} (1 - \delta_{2r}) \cdot \left( \tfrac{\left(1 + (\tfrac{3}{2} + 2|\mu|)\tfrac{1}{10^3}\right)^2}{10^3} + \kappa \left(2+\left(\tfrac{3}{2} + 2|\mu|\right) \cdot \tfrac{1}{10^3}\right) \left(\tfrac{3}{2} + 2|\mu|\right) \cdot \tfrac{1}{10^3} \right)\\
	  &\leq  0.0047 \cdot (1 - \delta_{2r}).
\end{align*}
\end{small}
where $(i)$ is by the definition of $\kappa := \tfrac{1 + \delta_{2r}}{1 - \delta_{2r}} \leq 1.223$ for $\delta_{2r} \leq \tfrac{1}{10}$, by assumption.
Combining the above in our main inequality, we obtain:
\begin{small}
\begin{align}
\|U_+ - U^\star R_Z\|_F^2 &\leq -0.1046 \cdot \eta^2 \cdot \|\mathcal{A}^\dagger(\mathcal{A}(ZZ^\top) - y) \cdot Z\|_F^2  \nonumber \\ 
					       &\quad \quad \quad \quad +\left(1 + \eta \sigma_r(X^\star) (1 - \delta_{2r}) \cdot (0.0047 - \sqrt{2} + 1)\right) \|Z - U^\star R_Z\|_F^2 \nonumber \\
					       &\leq \left(1 - \tfrac{4\eta \sigma_r(X^\star)(1 - \delta_{2r})}{10} \right)\|Z - U^\star R_Z\|_F^2 \label{eq:0000}
\end{align}
\end{small}
By Lemma \ref{lem:equiveta}, we know that $\eta \geq \tfrac{100}{102} \eta^\star$.
Also, $\eta^\star = \tfrac{1}{4 (1 + \delta_{2r}) \|X^\star\|_2}$, since $\|\mathcal{A}^\dagger(\mathcal{A}(X^\star) - y)\|_2 = 0$, in the noiseless setting.
Returning to \eqref{eq:0000}, we have:
\begin{small}
\begin{align*}
\|U_+ - U^\star R_Z\|_F^2 &\leq \left(1 - 0.393 \cdot \tfrac{(1 - \delta_{2r})\sigma_r(X^\star)}{(1 + \delta_{2r})\sigma_1(X^\star)} \right)\|Z - U^\star R_Z\|_F^2 \\
					       &= \left(1 - \tfrac{0.393}{\kappa \tau(X^\star)} \right)\|Z - U^\star R_Z\|_F^2
\end{align*}
\end{small}
Taking square root on both sides, we obtain:
\begin{small}
\begin{align*}
\|U_+ - U^\star R_Z\|_F &\leq \sqrt{1 - \tfrac{0.393}{\kappa \tau(X^\star)} } \|Z - U^\star R_Z\|_F
\end{align*}
\end{small}
Let us define $\xi = \sqrt{1 - \tfrac{0.393}{\kappa \tau(X^\star)} } < 1$. 
Using the definitions $Z = U + \mu (U - U_-)$ and $R_{Z} \in \arg\min\limits_{R \in \mathcal{O}} \|Z - U^{*}R\|_{F}$, we get
\begin{small}
\begin{align}
\|U_+ - U^\star R_Z\|_F &\leq \xi \cdot \min_{R \in \mathcal{O}} \|Z - U^\star R\|_F = \xi \cdot \min_{R \in \mathcal{O}} \|U + \mu (U - U_-) - U^\star R\|_F \nonumber \\
				&= \xi \cdot \min_{R \in \mathcal{O}} \|U + \mu \left(U - U_-\right)  - (1 - \mu + \mu) U^\star R\|_F \nonumber \\
				&\stackrel{(i)}{\leq} \xi \cdot |1 + \mu| \cdot \min_{R \in \mathcal{O}} \|U - U^\star R \|_F + \xi \cdot |\mu| \cdot \min_{R \in \mathcal{O}}\|U_- - U^\star R\|_F + \xi \cdot |\mu| \cdot r \sigma_1(X^\star)^{1/2} \nonumber 
\end{align}
\end{small}
where $(i)$ follows from steps similar to those in Lemma \ref{lem:supp_02}. 
Further observe that $\min_{R \in \mathcal{O}}\|U_+ - U^\star R\|_F \leq \|U_+ - U^\star R_Z\|_F$, thus leading to:
\begin{small}
\begin{align}
\min_{R \in \mathcal{O}}\|U_+ - U^\star R\|_F \leq \xi \cdot|1 + \mu| \cdot \min_{R \in \mathcal{O}} \|U - U^\star R \|_F + \xi \cdot |\mu| \cdot \min_{R \in \mathcal{O}}\|U_- - U^\star R\|_F  + \xi \cdot |\mu| \cdot r \sigma_1(X^\star)^{1/2} \label{eq:decrease}
\end{align}
\end{small}
Including two subsequent iterations in a single two-dimensional first-order system, we get the following characterization: 
\begin{small}
\begin{align*}
\begin{bmatrix}
\min_{R \in \mathcal{O}} \|U_{i+1} - U^\star R\|_F \\ \min_{R \in \mathcal{O}} \|U_i - U^\star R\|_F
\end{bmatrix} &\leq \underbrace{\begin{bmatrix}
\xi \cdot |1 + \mu| & \xi \cdot  |\mu| \\
1 & 0
\end{bmatrix}}_{:= A} \cdot 
\begin{bmatrix}
\min_{R \in \mathcal{O}} \|U_i - U^\star R \|_F \\
\min_{R \in \mathcal{O}} \|U_{i-1} - U^\star R \|_F
\end{bmatrix}
\\
&\quad \quad \quad \quad + 
\begin{bmatrix}
1 \\
0
\end{bmatrix} \cdot \xi \cdot |\mu| \cdot \sigma_1(X^\star)^{1/2} \cdot r,
\end{align*}
\end{small}
Observe that the contraction matrix $A$ has non-negative values.
Unfolding the above recursion for $J+1$ iterations, we obtain:
\begin{small}
\begin{align*}
\begin{bmatrix}
\min_{R \in \mathcal{O}} \|U_{J+1} - U^\star R\|_F \\ 
\min_{R \in \mathcal{O}} \|U_J - U^\star R\|_F
\end{bmatrix} &\leq A^{J+1} \cdot 
\begin{bmatrix}
\min_{R \in \mathcal{O}} \|U_0 - U^\star R\|_F \\
\min_{R \in \mathcal{O}} \|U_{-1} - U^\star R\|_F
\end{bmatrix} \\
&\quad \quad \quad \quad 
+ \left(\sum_{i = 0}^J A^i\right) \cdot 
\begin{bmatrix}
1 \\
0
\end{bmatrix} \cdot \xi \cdot |\mu| \cdot \sigma_1(X^\star)^{1/2} \cdot r
\end{align*}
\end{small}
Let us focus on the properties of the matrix $A$.
We re-use Lemma 2 in \cite{khanna2017iht}, after appropriate changes:

\begin{lemma}{\label{lemma:nth-power0}}
Let $A$ be the $2 \times 2$ matrix, as defined above, parameterized by $0 < \xi = \sqrt{1 - \tfrac{0.393}{\kappa \tau(X^\star)} } < 1$, and user-defined parameter $\mu$. 
The characteristic polynomial of $A$ is defined as:
\begin{align*}
\lambda^2 - \textup{\text{Tr}}(A) \cdot \lambda + \textup{\text{det}}(A) = 0
\end{align*}
where $\lambda$ represent the eigenvalue(s) of $A$.
Define $\Delta := \textup{\text{Tr}}(A)^2 - 4 \cdot \textup{\text{det}}(A) = \xi^2 \cdot (1 + \mu)^2 + 4 \xi \cdot |\mu|$.
Then, the eigenvalues of $A$ satisfy the expression: $\lambda_{1,2} = \tfrac{\xi \cdot |1 + \mu| \pm \sqrt{\Delta}}{2}$.
\end{lemma}
A proof of Lemma \ref{lemma:nth-power0} follows immediately from the quadratic formula. 
In our case, $\Delta > 0$, which means that 
\begin{align}{\label{eq:D}}
\lambda_{1,2} = \frac{\sqrt{1 - \tfrac{0.393}{\kappa \tau(X^\star)} } \cdot |1 + \mu| \pm \sqrt{\left(1 - \tfrac{0.393}{\kappa \tau(X^\star)}\right)  (1 + \mu)^2 + 4 \sqrt{1 - \tfrac{0.393}{\kappa \tau(X^\star)} } |\mu|}}{2}.
\end{align}
The following lemma describes how one can compute a power of a $2 \times 2$ matrix $A$, $A^i$, through the eigenvalues $\lambda_{1, 2}$ (real and distinct eigenvalues); the proof is provided in \cite{khanna2017iht}.
\begin{lemma}[\cite{williams1992nth}]{\label{lemma:nth-power2}}
Let $A$ be a $2 \times 2$ matrix with real eigenvalues $\lambda_{1,2}$. 
Then, the following expression holds, when $\lambda_1 \neq \lambda_2$:
\begin{align*}
A^i = \frac{\lambda_1^i - \lambda_2^i}{\lambda_1 - \lambda_2} \cdot A - \lambda_1 \lambda_2 \cdot \frac{\lambda_1^{i-1} - \lambda_2^{i-1}}{\lambda_1 - \lambda_2} \cdot I
\end{align*} 
where $\lambda_i$ denotes the $i$-th eigenvalue of A in order.
\end{lemma}
For the rest of the proof, we will require $|\lambda_i(A)| < 1$; we will discuss later in the text what this requirement means in terms of the parameters of the problem. 
Under this requirement, $\sum_{i = 0}^J A^i$ converges to: $\sum_{i = 0}^J A^i = \left(I - A\right)^{-1} \left(I - A^{J+1}\right)$, where:
\begin{align*}
B := \left(I - A\right)^{-1} \stackrel{|\mu| = \mu}{=} \tfrac{1}{1 - \xi(1 + 2\mu)} \cdot \begin{bmatrix} 1 & \xi \mu \\ 1 & 1 - \xi (1 + \mu) \end{bmatrix}
\end{align*}
This transforms our recursion into:
\begin{small}
\begin{align*}
&\begin{bmatrix}
\min_{R \in \mathcal{O}} \|U_{J+1} - U^\star R\|_F \\ 
\min_{R \in \mathcal{O}} \|U_J - U^\star R\|_F
\end{bmatrix} \\
&\leq A^{J+1} \cdot 
\begin{bmatrix}
\min_{R \in \mathcal{O}} \|U_0 - U^\star R\|_F \\
\min_{R \in \mathcal{O}} \|U_{-1} - U^\star R\|_F
\end{bmatrix} 
- B \cdot A^{J+1}  \cdot 
\begin{bmatrix}
1 \\
0
\end{bmatrix} \cdot \xi \cdot |\mu| \cdot \sigma_1(X^\star)^{1/2} \cdot r \\
&\quad \quad \quad \quad 
+ B \cdot 
\begin{bmatrix}
1 \\
0
\end{bmatrix} \cdot \xi \cdot |\mu| \cdot \sigma_1(X^\star)^{1/2} \cdot r \\
&\stackrel{(i)}{\leq} \frac{2|\lambda_1|^{J+1}}{|\lambda_1| - |\lambda_2|} \cdot A \cdot 
\begin{bmatrix}
\min_{R \in \mathcal{O}} \|U_0 - U^\star R\|_F \\
\min_{R \in \mathcal{O}} \|U_{-1} - U^\star R\|_F
\end{bmatrix} 
+ |\lambda_1| \cdot \frac{2|\lambda_1|^{J}}{|\lambda_1| - |\lambda_2|} \cdot 
\begin{bmatrix}
\min_{R \in \mathcal{O}} \|U_0 - U^\star R\|_F \\
\min_{R \in \mathcal{O}} \|U_{-1} - U^\star R\|_F
\end{bmatrix}  \\
&+  \frac{2|\lambda_1|^{J+1}}{|\lambda_1| - |\lambda_2|} \cdot  B \cdot A \cdot 
\begin{bmatrix}
1 \\
0
\end{bmatrix} \cdot \xi \cdot |\mu| \cdot \sigma_1(X^\star)^{1/2} \cdot r
+  |\lambda_1| \cdot \frac{2|\lambda_1|^{J}}{|\lambda_1| - |\lambda_2|} \cdot 
B \cdot 
\begin{bmatrix}
1 \\
0
\end{bmatrix} \cdot \xi \cdot |\mu| \cdot \sigma_1(X^\star)^{1/2} \cdot r \\
&\quad \quad \quad \quad 
+ B \cdot 
\begin{bmatrix}
1 \\
0
\end{bmatrix} \cdot \xi \cdot |\mu| \cdot \sigma_1(X^\star)^{1/2} \cdot r \\
&\stackrel{(ii)}{\leq} \frac{2|\lambda_1|^{J+1}}{|\lambda_1| - |\lambda_2|} \cdot (A + I) \cdot 
\begin{bmatrix}
\min_{R \in \mathcal{O}} \|U_0 - U^\star R\|_F \\
\min_{R \in \mathcal{O}} \|U_{-1} - U^\star R\|_F
\end{bmatrix} 
+ \frac{2|\lambda_1|^{J+1}}{|\lambda_1| - |\lambda_2|} \cdot  B \cdot (A + I) \cdot 
\begin{bmatrix}
1 \\
0
\end{bmatrix} \cdot \xi \cdot |\mu| \cdot \sigma_1(X^\star)^{1/2} \cdot r \\
&\quad \quad \quad \quad 
+ B \cdot 
\begin{bmatrix}
1 \\
0
\end{bmatrix} \cdot \xi \cdot |\mu| \cdot \sigma_1(X^\star)^{1/2} \cdot r \\
&\stackrel{(iii)}{\leq} \frac{2|\lambda_1|^{J+1}}{|\lambda_1| - |\lambda_2|} \cdot \begin{bmatrix}
1+ \xi \cdot |1 + \mu| & \xi \cdot  |\mu| \\
1 & 1
\end{bmatrix} \cdot 
\begin{bmatrix}
\min_{R \in \mathcal{O}} \|U_0 - U^\star R\|_F \\
\min_{R \in \mathcal{O}} \|U_{-1} - U^\star R\|_F
\end{bmatrix} \\
&\quad \quad \quad \quad + \frac{2|\lambda_1|^{J+1}}{|\lambda_1| - |\lambda_2|} \cdot 
\tfrac{1}{1 - \xi(1 + 2\mu)} \cdot \begin{bmatrix}
1 + \xi(1 + 2\mu) \\
2
\end{bmatrix} \cdot \xi \cdot |\mu| \cdot \sigma_1(X^\star)^{1/2} \cdot r \\
&\quad \quad \quad \quad \quad \quad 
+ \tfrac{1}{1 - \xi(1 + 2\mu)} \cdot 
\begin{bmatrix}
1 \\
1
\end{bmatrix} \cdot \xi \cdot |\mu| \cdot \sigma_1(X^\star)^{1/2} \cdot r 
\end{align*}
\end{small}
where for $(i)$ we used Lemma \ref{lemma:nth-power2}, and the fact that $B$ and $A$ have non-negative values to retain the bound,
$(ii)$ is due to $|\lambda_1| < 1$,
$(iii)$ is obtained by explicitly computing \begin{tiny}$B A \begin{bmatrix} 1 \\ 0 \end{bmatrix}$\end{tiny} and \begin{tiny}$B  \begin{bmatrix} 1 \\ 0 \end{bmatrix}$\end{tiny}.

Focusing on the top term of this expression and under the assumption that $\min_{R \in \mathcal{O}} \|U_0 - U^\star R\|_F = \min_{R \in \mathcal{O}} \|U_{-1} - U^\star R\|_F$, we have:
\begin{small}
\begin{align*}
\min_{R \in \mathcal{O}} \|U_{J+1} - U^\star R\|_F &\leq \frac{2|\lambda_1|^{J+1}}{|\lambda_1| - |\lambda_2|} \cdot \left((1 + \xi(1 + 2\mu))\min_{R \in \mathcal{O}} \|U_0 - U^\star R\|_F +  \tfrac{\xi(1 + \xi(1 + 2\mu))}{1 - \xi(1 + 2\mu)} \cdot |\mu| \cdot \sigma_1(X^\star)^{1/2} \cdot r \right) \\ &\quad \quad \quad \quad \quad \quad + \tfrac{\xi \mu}{1 - \xi(1 + 2\mu)} \cdot  \sigma_1(X^\star)^{1/2} \cdot r
\end{align*}
\end{small}
We will use the assumption on $\mu$; then
\begin{align*}
 \tfrac{\xi(1 + \xi(1 + 2\mu))}{1 - \xi(1 + 2\mu)} \cdot |\mu| \cdot \sigma_1(X^\star)^{1/2} \cdot r \leq  \tfrac{\xi(1 + \xi(1 + \tfrac{1}{10^3}))}{1 - \xi(1 + \tfrac{1}{10^3})} \cdot \tfrac{\varepsilon \cdot \sigma_r(X^\star)^{1/2}}{2\cdot 10^3 \sqrt{\kappa \tau(X^\star)}}.
\end{align*}
and
\begin{align*}
\tfrac{\xi |\mu|}{1 - \xi(1 + 2\mu)} \cdot  \sigma_1(X^\star)^{1/2} \cdot r \leq \tfrac{\xi}{1 - \xi(1 + \tfrac{1}{10^3})} \cdot \tfrac{\varepsilon \cdot \sigma_r(X^\star)^{1/2}}{2\cdot 10^3 \sqrt{\kappa \tau(X^\star)}}
\end{align*}
which further leads to:
\begin{small}
\begin{align*}
\min_{R \in \mathcal{O}} \|U_{J+1} - U^\star R\|_F &\leq \frac{2|\lambda_1|^{J+1}}{|\lambda_1| - |\lambda_2|} \cdot \left((1 + \xi(1 + 2\mu))\min_{R \in \mathcal{O}} \|U_0 - U^\star R\|_F +  \tfrac{\xi(1 + \xi(1 + \tfrac{1}{10^3}))}{1 - \xi(1 + \tfrac{1}{10^3})} \cdot \tfrac{\varepsilon \cdot \sigma_r(X^\star)^{1/2}}{2\cdot 10^3 \sqrt{\kappa \tau(X^\star)}} \right) \\ &\quad \quad \quad \quad \quad \quad + \tfrac{\xi}{1 - \xi(1 + \tfrac{1}{10^3})} \cdot \tfrac{\varepsilon \cdot \sigma_r(X^\star)^{1/2}}{2\cdot 10^3 \sqrt{\kappa \tau(X^\star)}}
\end{align*}
\end{small}
Under the assumption that $\min_{R \in \mathcal{O}} \|U_0 - U^\star R\|_F \leq \tfrac{ \sigma_r(X^\star)^{1/2}}{10^3 \sqrt{\kappa \tau(X^\star)}}$, we get the expression:
\begin{small}
\begin{align*}
\min_{R \in \mathcal{O}} \|U_{J+1} - U^\star R\|_F &\leq \frac{2|\lambda_1|^{J+1}}{|\lambda_1| - |\lambda_2|} \cdot \left((1 + \xi(1 + 2\mu)) +   \tfrac{\xi(1 + \xi(1 + \tfrac{1}{10^3}))}{1 - \xi(1 + \tfrac{1}{10^3})}\cdot \tfrac{\varepsilon}{2} \right) \cdot \tfrac{\sigma_r(X^\star)^{1/2}}{10^3 \sqrt{\kappa \tau(X^\star)}} \\ 
&\quad \quad \quad \quad + \tfrac{\xi \cdot \varepsilon}{2(1 - \xi(1 + \tfrac{1}{10^3}))} \cdot \tfrac{\sigma_r(X^\star)^{1/2}}{10^3 \sqrt{\kappa \tau(X^\star)}} \\
&\lesssim c^{J+1} \cdot \min_{R \in \mathcal{O}} \|U_0 - U^\star R\|_F + O\left(\mu\right),
\end{align*}
\end{small}
for a constant $c < 1$.
In words, the proposed algorithm achieves a linear convergence rate in iterate distances (first term on RHS), up to a constant proportional to the the momentum hyperparameter $\mu$ (second term on RHS).

The above hold assuming $|\lambda_i(A)| < 1$; we now focus on the requirement that $|\lambda_i(A)| < 1$. 
Based on the expression in \eqref{eq:D}, and using the facts that: $i)$ for $\delta_{2r} \leq \tfrac{1}{10}$, we have $\kappa \leq 1.223$, and $ii)$ by assumption $\mu \leq \tfrac{1}{10^3}$, we observe that $\lambda_i(A) < 1$ for $\tau(X^\star) \leq 78$; see also Figure \ref{fig:lambdas}. 
\begin{figure}[!ht]
\centering
\includegraphics[width=0.45\textwidth]{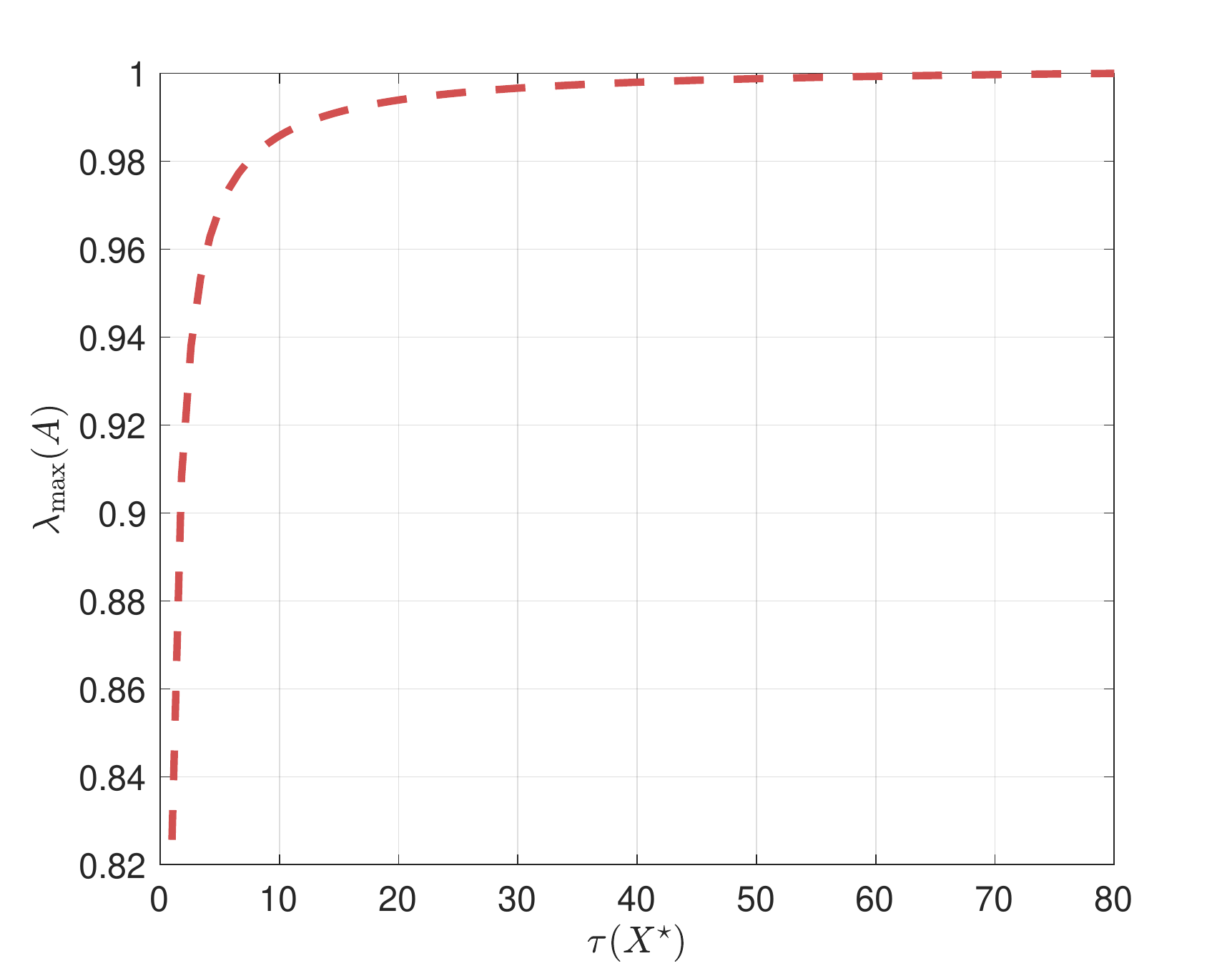} 
\includegraphics[width=0.45\textwidth]{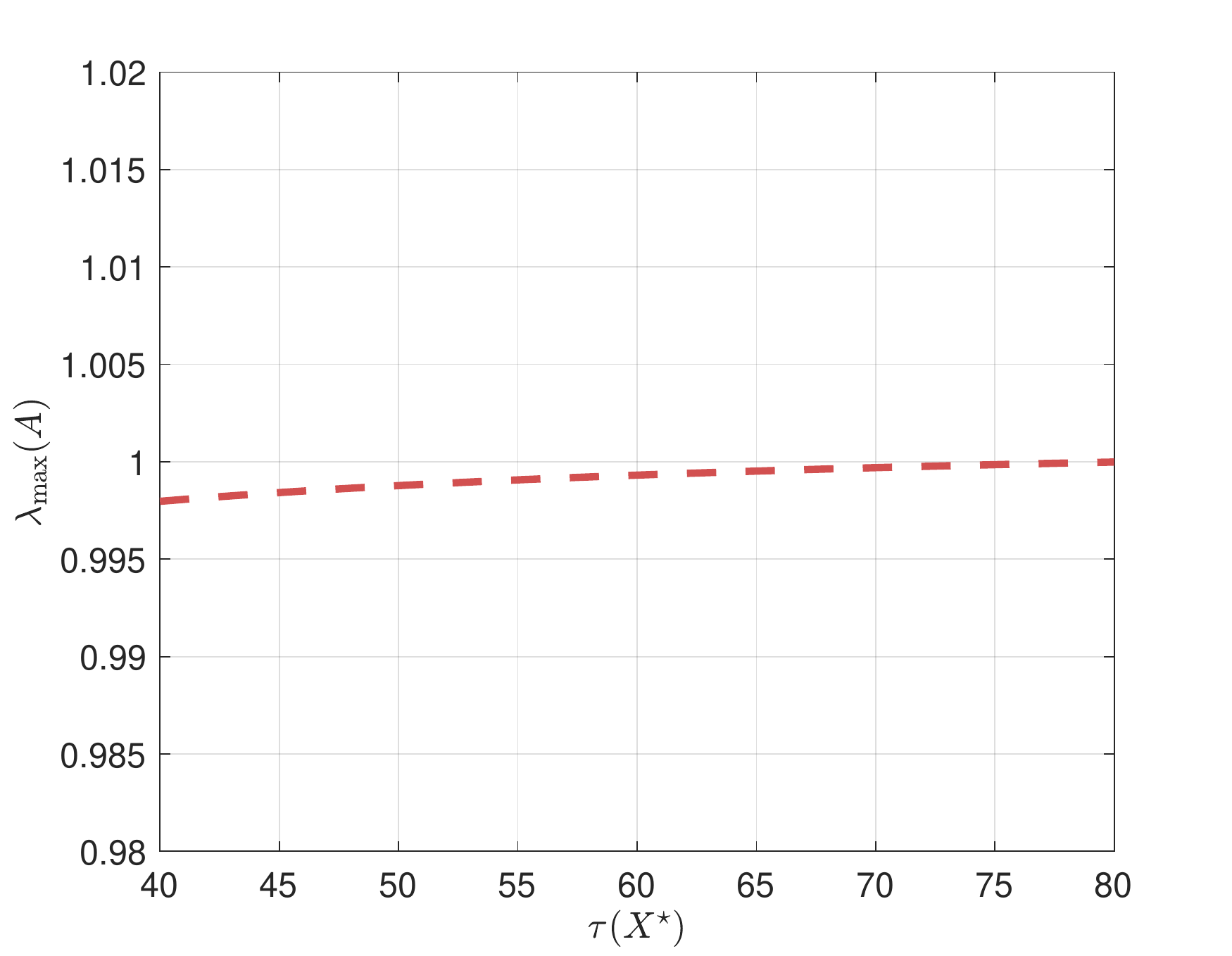} 
 \caption{Behavior of maximum $|\lambda_{i}(A)|$ for various values of $\tau(X^\star)$; right plot is a zoomed-in version of the left plot.}\label{fig:lambdas}
\end{figure} 
For our theory next, we will assume $\tau(X^\star) \leq 50$, so that the following requirement holds:
In \eqref{eq:decrease}, we observe that if $\min_{R \in \mathcal{O}} \|U - U^\star R\|_F \leq \tfrac{\sigma_r(X^\star)^{1/2}}{10^3 \sqrt{\kappa\tau(X^\star)}}$ and $\min_{R \in \mathcal{O}} \|U_- - U^\star R\|_F \leq \tfrac{\sigma_r(X^\star)^{1/2}}{10^3 \sqrt{\kappa\tau(X^\star)}}$, and under the assumptions on $\mu$, we have:
\begin{small}
\begin{align*}
\min_{R \in \mathcal{O}} \|U_+ - U^\star R\|_F &\leq \xi \cdot |1 + \mu| \cdot \tfrac{\sigma_r(X^\star)^{1/2}}{10^3 \sqrt{\kappa\tau(X^\star)}} + \xi \cdot |\mu| \cdot \tfrac{\sigma_r(X^\star)^{1/2}}{10^3 \sqrt{\kappa\tau(X^\star)}} + \xi |\mu| \cdot r \sigma_1(X^\star)^{1/2} \\
&\leq \xi \left(1 + \tfrac{3}{10^3} \right) \tfrac{\sigma_r(X^\star)^{1/2}}{10^3 \sqrt{\kappa\tau(X^\star)}} \leq \tfrac{\sigma_r(X^\star)^{1/2}}{10^3 \sqrt{\kappa\tau(X^\star)}}
\end{align*}
\end{small}
for $\xi < 0.9968$ for $\tau(X^\star) \leq 50$, and for any $\varepsilon \in (0,1)$. 
\emph{I.e.}, $U_+$ satisfies $\min_{R \in \mathcal{O}} \|U_+ - U^\star R\|_F \leq \tfrac{\sigma_r(X^\star)^{1/2}}{10^3 \sqrt{\kappa\tau(X^\star)}}$.
Since the distance remains bounded after each iteration, Lemma 6 hold for all $i$.

\section{$\mu$ECoG simulation details}
In section~\ref{sec:Neuro}, we presented a novel Neuroscience application for 
low rank matrix sensing and our accelerated Procrustes flow algorithm. 
We presented results for recovering individual neuron activities from 
stimulus evoked cortical surface electrical potentials ($\mu$ECoG).
Here, we give additional details and results related to this experiment.

For the simulation, we considered a spiking neural network model with $1000$ neurons
($200$ inhibitory and $800$ excitatory neurons).
We then simluted 20 seconds of spiking activities for these neurons and sampled these activities 
at a sampling rate of $200Hz$. 
The input stimulus occurred every $2$ seconds and lasted $0.3$secs.
Hence, the (unknown) neuronal activity matrix $X$ was of size $1000\times 4200$. 
The neurons are assumed to be uniformly distributed along the depth between 1-210$\mu$m from
the surface. 

The surface potentials $y$ (single $\mu$ECoG electrode recording) 
was then computed from these neuronal potentials using the distance
dependent lowpass filtering, where the cutoff frequency [$f_c(d)$] is defined as in sec.~\ref{sec:Neuro},
 the amplitude attenuated according to the distance, and then summing up the potentials.
 We chose the distance parameters to be
$\Delta_1=1,\Delta_2=0.25,h=10$. Note that, these operations (low pass filtering and attenuation)
can be combined into a linear transformation $\mathcal{A}(X)=\mbf{A}\tvec(X)$ on the neuronal activity matrix $X$.
 The matrix $\mbf{A}\in\RR^{m\times nm}$ is a banded matrix
assuming an FIR filter for  lowpass filtering. In our case, $m \cdot n=4.2\times10^6$.
The objective of the low rank matrix sensing model is to recover the simulated
neuronal activity matrix $X$ from the surface potentials $y$.

 \begin{figure}[tb!]
 \begin{center}
\includegraphics[width=1.0\textwidth,,trim={3.0cm 0cm  3.0cm 0cm},clip]{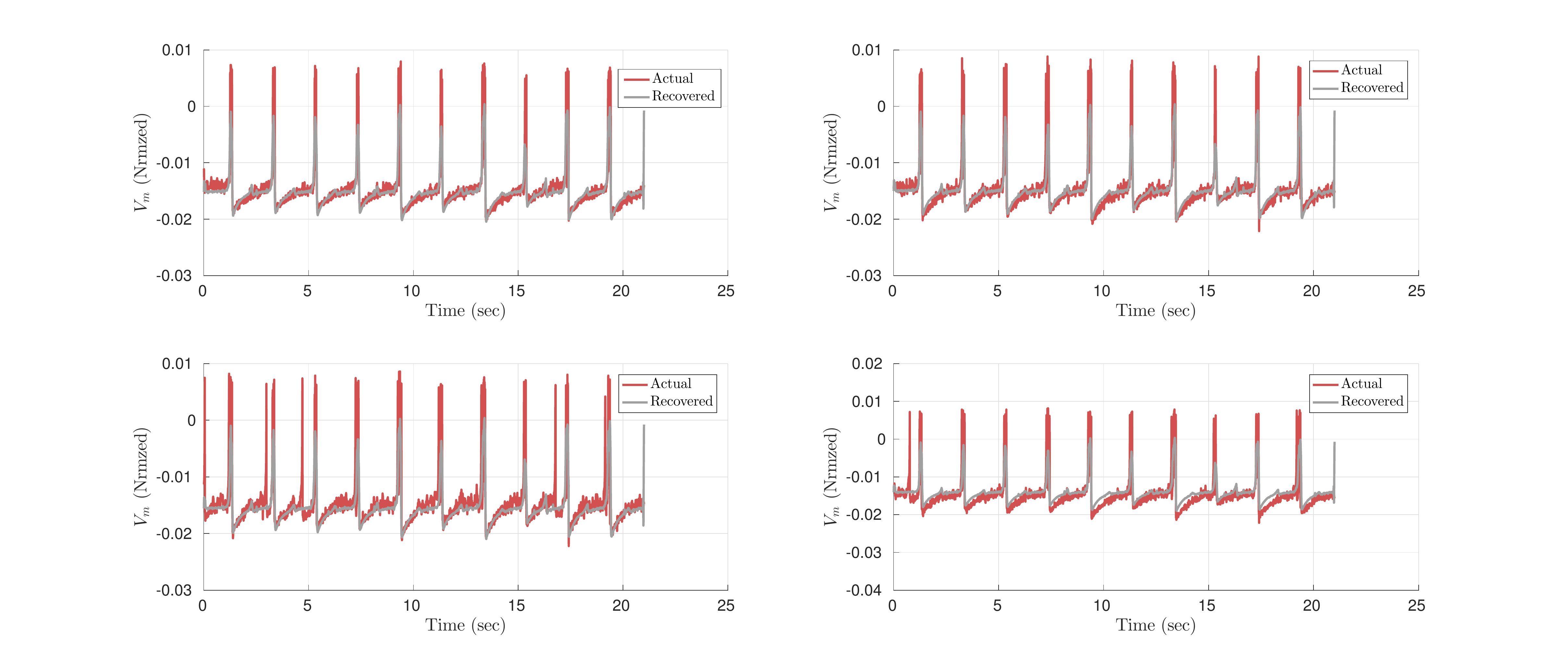}
\vskip -0.2in
 \caption{The actual membrane potentials during the 20secs simulation
 and the recovered potentials for the four neurons.
 }\label{fig:Neuro20}
 \end{center}
\end{figure}

 \begin{figure}[tb!]
 \begin{center}
\includegraphics[width=0.45\textwidth]{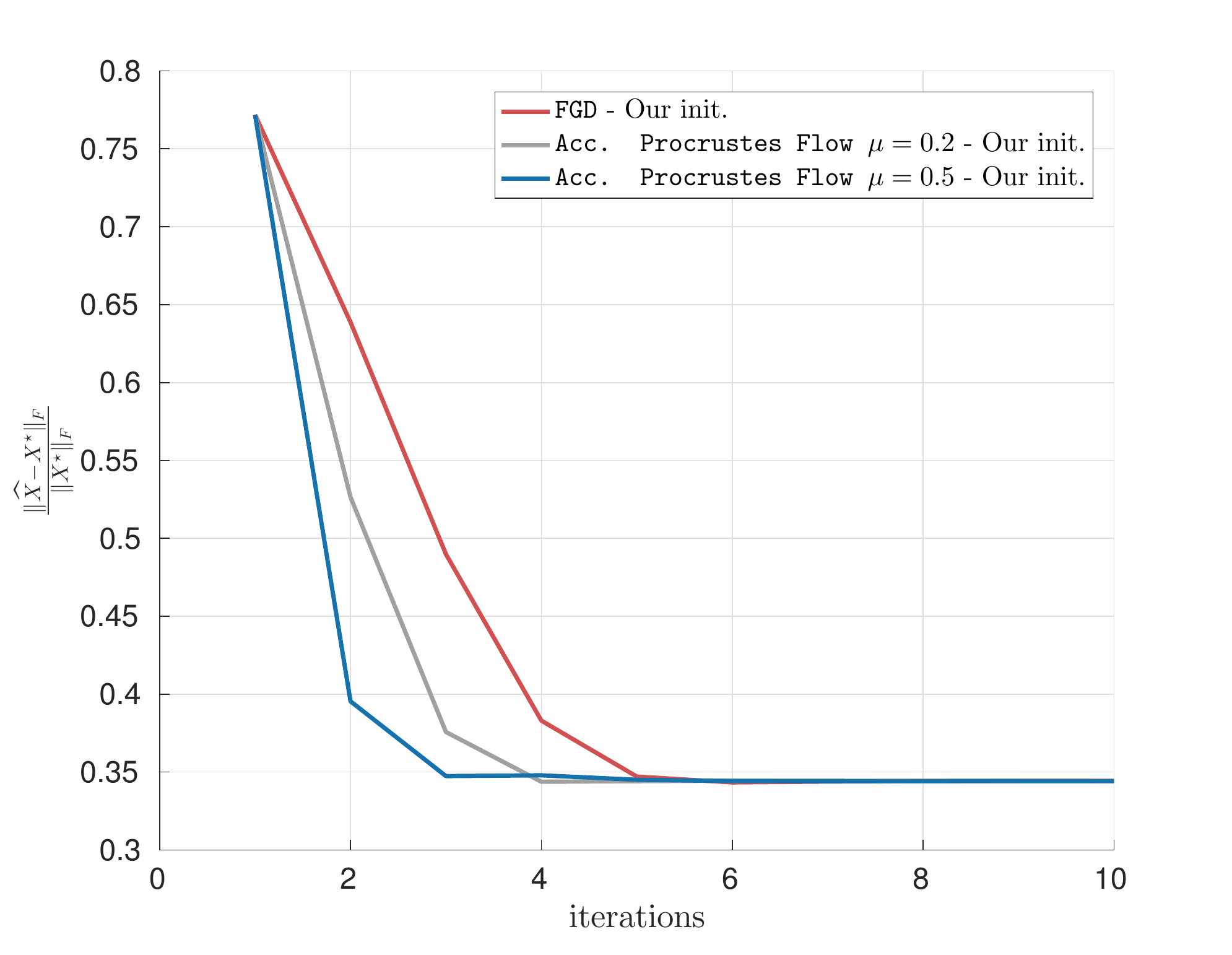}
\includegraphics[width=0.45\textwidth]{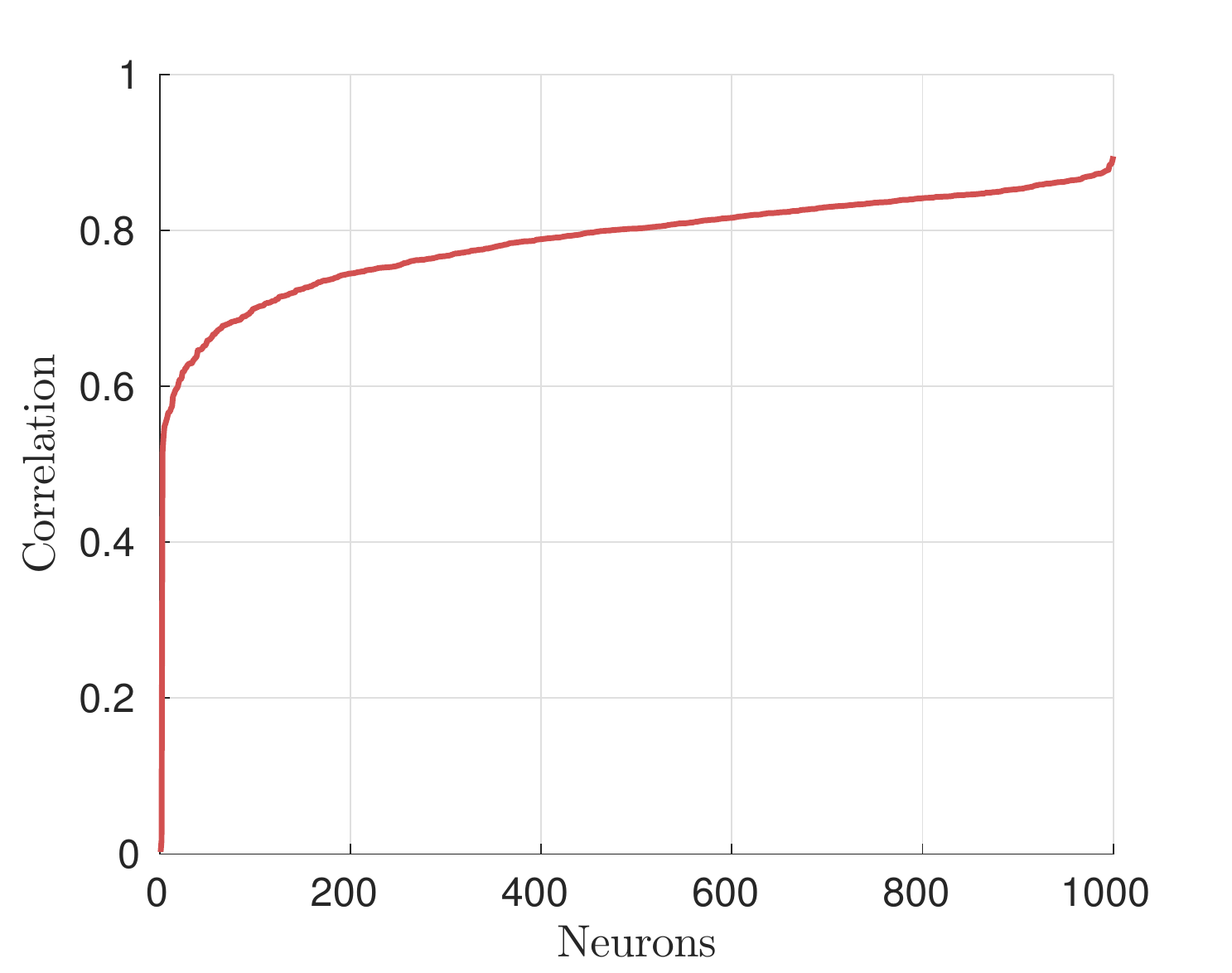}
\vskip -0.1in
 \caption{Neuronal activity recovery:  Convergence behavior vs. number of iterations (left),
 Correlation between the recovered and actual membrane potentials for the 1000 neurons (right).
 }\label{fig:NeuroConv}
 \end{center}
\end{figure}

The neuronal activity matrix $X$ is low rank  with $\texttt{rank}(X) \leq r$,
and can be written as $X = UV^\top$, for $U \in \mathbb{R}^{n \times r},
V\in \mathbb{R}^{m \times r}$. The factorized low rank matrix sensing problem becomes
\begin{equation}\label{eq:bifactobj}
\min_{U \in \mathbb{R}^{n \times r}, V\in \mathbb{R}^{m\times r}} ~\tfrac{1}{2} \|\mathcal{A}(UV^\top) - y\|_2^2.
\end{equation}
Rectangular version of the Procrustes flow algorithm has been studied in \cite{ tu2016low,park2016non, hsieh2017non} to solve the above problem. Here, we consider the following accelerated 
Procrustes flow:
\begin{align*}
U_{i+1} &= Z_{i} - \eta \mathcal{A}^\dagger \left(\mathcal{A}(Z_i W_i^\top) - y\right) \cdot W_i, \\
V_{i+1} &= W_{i} - \eta \mathcal{A}^\dagger \left(\mathcal{A}(Z_i W_i^\top) - y\right)^\top \cdot Z_i, \\
Z_{i+1} &= U_{i+1} + \mu \left(U_{i+1} - U_i\right),\\
W_{i+1} &= V_{i+1} + \mu \left(V_{i+1} - V_i\right).
\end{align*}
$Z_i,W_i$ are the auxiliary variables that accumulate the ``momentum" of
the variables $U,V$; the dimensions are apparent from the context. 
$\mu$ is the momentum parameter that weighs how the previous estimates $U_i,V_i$ will 
be mixed with 
the current estimate $U_{i+1},V_{i+1}$ to generate $Z_{i+1},W_{i+1}$.
The parameters $\eta$ and $\mu$ and the inital points are selected as in Algorithm~\ref{alg:algo1}.
The rank $r=10$ for our data.

Figure~\ref{fig:Neuro20} plots the recovered potentials (along with the actual membrane potentials) for the four neurons (4 rows of $\widehat{X}=\widehat{U}\widehat{V}^\top$ and $X$),
corresponding to the four neurons discussed in sec.~\ref{sec:Neuro} over 20 secs of the simulation.
The convergence behavior of factored gradient descent (rectangluar Procrustes flow) and our 
accelerated Procrustes flow (with $\mu=0.2$ and $\mu=0.5$) algorithms for the neural activity recovery problem 
 are given in the left plot of fig.~\ref{fig:NeuroConv}.
We also give the correlation between the recovered and actual membrane potentials for the 1000 neurons in
the right plot. We note that for most neurons ($\sim90\%$),  the recovered potentials are close to the actual simulated ones.

\end{document}